\documentclass{article} %
\usepackage{paper_format,times}

\usepackage{amsmath,amsfonts,bm}

\def\eqref#1{equation~\ref{#1}}

\def\1{\bm{1}}

\DeclareMathAlphabet{\mathsfit}{\encodingdefault}{\sfdefault}{m}{sl}
\SetMathAlphabet{\mathsfit}{bold}{\encodingdefault}{\sfdefault}{bx}{n}

\usepackage{float}

\usepackage[hidelinks]{hyperref}
\usepackage{url}
\usepackage{graphicx}
\usepackage{wrapfig}
\usepackage{amsthm}
\usepackage{booktabs}
\usepackage{multirow}
\usepackage{xcolor}
\usepackage{colortbl}
\usepackage{enumitem}
\setlist[itemize]{topsep=2pt,itemsep=3pt,parsep=0pt,partopsep=0pt}
\usepackage{array}
\usepackage{tcolorbox}
\usepackage{pifont}
\usepackage{listings}
\usepackage{tablefootnote}
\tcbuselibrary{skins,breakable,listings}

\floatstyle{ruled}
\newfloat{algorithm}{tbp}{loa}
\floatname{algorithm}{Algorithm}

\newcounter{algline}
\newlength{\algindent}
\newlength{\algcur}
\makeatletter
\newenvironment{algorithmic}[1][0]{%
  \setcounter{algline}{0}%
  \setlength{\algcur}{0pt}%
  \par\footnotesize
  \setlength{\parindent}{0pt}%
  \newcommand{\algline@print}{%
    \stepcounter{algline}%
    \makebox[1.8em][r]{\scriptsize\thealgline:}\hspace{0.5em}}%
  \newcommand{\State}{\par\noindent\algline@print\hspace{\algcur}\ignorespaces}%
  \newcommand{\algin}{\addtolength{\algcur}{\algindent}}%
  \newcommand{\algout}{\addtolength{\algcur}{-\algindent}}%
  \newcommand{\Require}{\par\noindent\makebox[1.8em][r]{}\hspace{0.5em}\textbf{Require:} \ignorespaces}%
  \newcommand{\Ensure}{\par\noindent\makebox[1.8em][r]{}\hspace{0.5em}\textbf{Ensure:} \ignorespaces}%
  \newcommand{\For}[1]{\State\textbf{for} ##1 \textbf{do}\algin}%
  \newcommand{\EndFor}{\algout\State\textbf{end for}}%
  \newcommand{\Return}{\textbf{return} \ignorespaces}%
}{\par}
\makeatother

\newtheorem{proposition}{Proposition}
\newtheorem{theorem}{Theorem}
\newtheorem{corollary}{Corollary}

\definecolor{ctrlgray}{RGB}{238,238,238}
\definecolor{oursblue}{RGB}{232,242,255}
\definecolor{labelgreen}{RGB}{234,245,243}

\definecolor{singlecapblue}{HTML}{3A82C4}
\definecolor{multicaporange}{HTML}{E7795B}

\newcommand{\ctrlc}[1]{\cellcolor{ctrlgray}{#1}}
\newcommand{\oursc}[1]{\cellcolor{oursblue}{#1}}
\newcommand{\labelc}[1]{\cellcolor{labelgreen}{#1}}
\newcommand{\std}[1]{{\,\scalebox{0.7}{\ensuremath{\pm}#1}}}

\newcolumntype{L}[1]{>{\raggedright\arraybackslash}p{#1}}

\newcommand{\cmark}{\ding{51}}
\newcommand{\xmark}{\ding{55}}
\definecolor{mathteal}{RGB}{58,130,196}
\definecolor{codeorange}{RGB}{231,121,91}
\definecolor{ifgreen}{RGB}{63,163,124}
\definecolor{okgreen}{RGB}{31,122,80}
\definecolor{badred}{RGB}{176,48,48}

\lstdefinestyle{pythonstyle}{
  language=Python,
  basicstyle=\ttfamily\scriptsize,
  keywordstyle=\color{blue!70!black}\bfseries,
  stringstyle=\color{red!60!black},
  commentstyle=\color{green!50!black}\itshape,
  breaklines=true,
  frame=none,
  resetmargins=true,
  xleftmargin=0pt,
  framexleftmargin=0pt,
  aboveskip=2pt,
  belowskip=2pt,
}

\title{Distill What You Trust: Reliability-Aware Multi-Teacher On-Policy Distillation}

\author{
Jie Sun$^{1,2,*}$, Mao Zheng$^{2,*}$, Mingyang Song$^{2,*}$, Zeyuan Liu$^{2,3}$, Gengsheng Li$^{2,4}$ \\
\textbf{Houcheng Jiang$^2$, Yilin Cheng$^5$, Bichuan Feng$^6$, Yuchen Cai$^2$, Junfeng Fang$^{7,\dagger}$, Xiang Wang$^{7,\dagger}$} \\
{$^1$Shanghai Innovation Institute \quad $^2$Foundation Model Department, Tencent} \\
{$^3$Tsinghua University \quad $^4$Institute of Automation, Chinese Academy of Sciences} \\
{$^5$Zhongguancun Academy \quad $^6$Nankai University \quad $^7$National University of Singapore} \\
{$^*$ Equal Contribution \quad $\dagger$ Corresponding authors}
}

\fancypagestyle{firstpagestyle}{%
  \fancyhf{}
  \fancyhead[L]{\includegraphics[height=0.66cm]{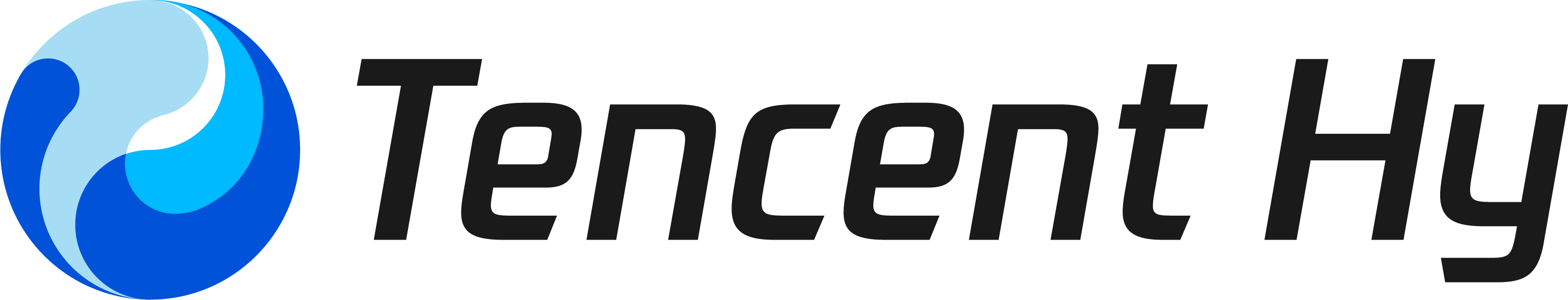}}
  \fancyhead[R]{\textcolor{gray}{\fontsize{11pt}{10pt}\selectfont October 11, 2026}}
  \fancyfoot[C]{\thepage}
  \renewcommand{\headrulewidth}{0.8pt}
  \renewcommand{\footrulewidth}{0pt}
}

\paperfinalcopy
\begin{document}

\maketitle
\thispagestyle{firstpagestyle}

\begin{abstract}
Multi-teacher on-policy distillation allows a student to learn from complementary specialists on its own trajectories. 
Domain-routed approaches, however, select one teacher per example and keep it fixed throughout the response. 
This design both depends on domain labels that mixed training corpora often lack and cannot adapt teacher selection when the expertise required changes within a trajectory.
We observe that each specialist deviates more from a shared reference on in-domain prompts than on out-of-domain prompts, on average.
Based on this observation, we propose \textbf{TrustMOPD}, which replaces example-level teacher selection with label-free, token-level supervision allocation.
At each student-generated prefix, TrustMOPD measures this displacement in next-token preferences, calibrates its magnitude across teachers, and uses the resulting scores as proxies for local reliability to weight teacher-specific distillation losses.
Evaluated across mathematics, code, and instruction following, TrustMOPD closes 91.5\% and 98.0\% of the overall-score gap between the initial student and oracle-routed teachers when trained on \textsc{SingleCap} and \textsc{MultiCap}, respectively, compared with 54.4\% and 54.5\% for the strongest label-free baseline in each setting.
On \textsc{SingleCap}, it approaches label-based MOPD without using domain labels.
Our code is available at \url{https://github.com/sunjie279/TrustMOPD-}.
\end{abstract}

\section{Introduction}

Integrating the complementary capabilities of multiple specialized models into a single student is a fundamental challenge in building broadly capable language models~\citep{wan2024fusellm,tian2025tinyllm}.
On-policy distillation (OPD) queries a teacher along trajectories generated by the student, providing supervision at states the student actually visits~\citep{lu2025onpolicydistillation,agarwal2024onpolicy_opsd,gu2024minillm_opd,ko2024distillm_opd}.
Multi-teacher OPD (MOPD) extends this framework to complementary specialists, enabling a student to acquire expertise in mathematical reasoning, code generation, and instruction following within a unified training process~\citep{ma2026mopd,openmopd2026}.
This raises a central question: how should supervision be allocated among teachers at each generation state?

Domain-routed MOPD approaches address this question at the example level~\citep{xiao2026mimov2flash}.
As illustrated in Figure~\ref{fig:teaser}(a), they select one teacher based on the example's domain label and keep that choice fixed throughout the response~\citep{ma2026mopd,openmopd2026}.
This design assumes that an available domain label identifies a suitable teacher and that the selected teacher remains appropriate throughout the response.
These assumptions may fail when domain labels are unavailable or supervision needs change within a response.
For instance, stronger mathematical reasoning does not necessarily imply better adherence to instruction constraints~\citep{fu-etal-2026-mathif}.
The illustrative prompt in Figure~\ref{fig:teaser}(a) requires mathematical derivation, code-based verification, and JSON formatting.
A mathematics specialist may provide useful guidance during derivation without being equally suited to code verification or output formatting.
Assigning one teacher to the full response cannot accommodate such changes in supervision needs.

More fundamentally, a teacher's global domain expertise does not guarantee local supervision reliability at the current generation state~\citep{he2026mtsdpo,wang2026demystifyingopd}. 
Global expertise describes what a teacher is good at, whereas local reliability concerns whether its learning signal at the current state is suitable and worth transferring. 
In OPD, local reliability must be assessed at student-generated prefixes rather than fixed reference prefixes~\citep{fu2026revisitingopd}.
Teacher-student disagreement is also insufficient as a reliability measure: a large discrepancy may reflect useful expertise, but it may instead arise from stylistic preferences or other idiosyncratic behavior that the student need not learn~\citep{wang2026teachability,li2026rethinkingopd}. 
Allocating supervision across specialists therefore requires a state-dependent reliability estimate that is comparable across teachers and does not rely on domain labels.

\begin{figure}[t]
\centering
\includegraphics[width=\linewidth]{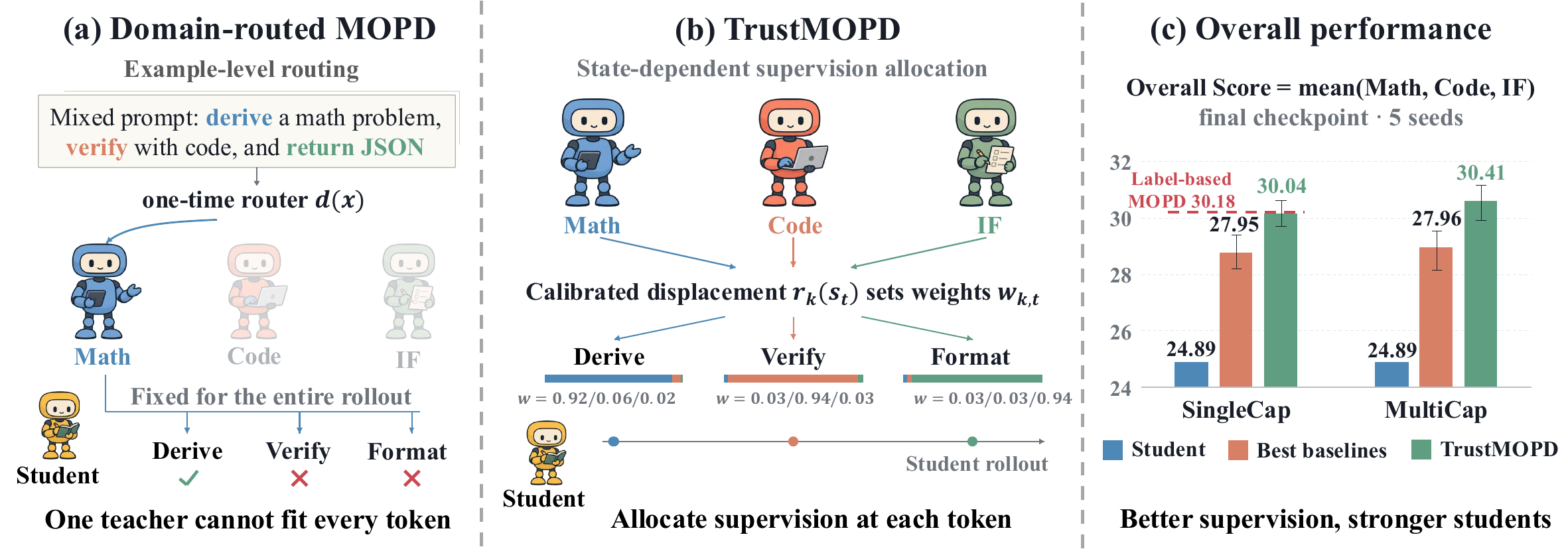}
\caption{\textbf{(a)} Domain-routed MOPD assigns one teacher per example for the full response, even when a mixed prompt requires different expertise for derivation, code verification, and output formatting.
\textbf{(b)} TrustMOPD uses calibrated teacher-reference displacement as a reliability proxy and converts the scores into token-level weights without domain labels. 
\textbf{(c)} TrustMOPD outperforms the strongest label-free baselines when trained on \textsc{SingleCap} and \textsc{MultiCap}, while approaching label-based MOPD on \textsc{SingleCap}.
Overall is the mean of mathematics, code, and instruction-following scores.
Results are means over five seeds; error bars denote standard deviations.}
\label{fig:teaser}
\end{figure}

To meet these requirements, we propose \textbf{TrustMOPD}, which allocates supervision using each teacher's displacement in next-token preferences relative to a shared reference, as illustrated in Figure~\ref{fig:teaser}(b).
Our empirical analysis in Figure~\ref{fig:allocation-analysis}(a) shows that each teacher's mean displacement is largest on prompts from its own domain.
However, raw magnitudes differ across teachers: on mathematics prompts, the code teacher's displacement exceeds that of the mathematics teacher.
We therefore divide each teacher's displacement by its fixed offline mean to calibrate the scores across teachers.
The resulting scores favor the domain-matched teacher on average, motivating their use as proxies for local supervision reliability.
At each student-generated prefix, TrustMOPD converts the calibrated scores into supervision weights, assigning greater weight to teachers with higher scores.
These weights combine teacher-specific distillation losses into a single training objective, allowing supervision allocation to adapt along the trajectory without domain labels or a learned router.
Our theoretical analysis relates teacher-reference displacement to changes in next-token log odds and provides a continuation-utility interpretation under an idealized KL-regularized optimum.

We evaluate TrustMOPD with three specialists in mathematics, code, and instruction following using two training sets: \textsc{SingleCap}, containing single-capability prompts, and \textsc{MultiCap}, containing both single- and multi-capability prompts. Figure~\ref{fig:teaser}(c) summarizes the overall performance.
On \textsc{SingleCap}, TrustMOPD closes $91.5\%$ of the overall-score gap between the initial student and oracle-routed teachers, versus $54.4\%$ for the strongest label-free baseline, approaching label-based MOPD without domain labels.
On \textsc{MultiCap}, where most prompts combine two capabilities and lack a unique domain label, TrustMOPD increases the recovery ratio from $54.5\%$ for the strongest label-free baseline to $98.0\%$. 
Further ablations indicate that calibration and token-level allocation are particularly beneficial when training prompts combine multiple capabilities.
Together, these results support label-free, state-dependent supervision allocation as an effective approach to integrating complementary expertise from multiple teachers.

\section{Preliminaries}
\label{sec:preliminaries}

On-policy distillation (OPD) trains a student on its own rollouts by querying the teacher at student-generated prefixes~\citep{lu2025onpolicydistillation,agarwal2024onpolicy_opsd,gu2024minillm_opd,ko2024distillm_opd}. Let $\mathcal{D}$ denote the input distribution, $\pi_\theta$ the student model, and $\{\pi_k\}_{k=1}^{K}$ a collection of $K$ teachers. Given an input $x\sim\mathcal{D}$, the student samples a variable-length response $y\sim\pi_\theta(\cdot\mid x)$, whose length is denoted by $|y|$. At position $t\in\{1,\ldots,|y|\}$, the student-generated state is $s_t=(x,y_{<t})$, where $y_{<t}=(y_1,\ldots,y_{t-1})$ is the generated prefix. We write $p_{\theta,t}=\pi_\theta(\cdot\mid s_t)$ and $q_{k,t}=\pi_k(\cdot\mid s_t)$ for the next-token distributions of the student and teacher $k$, respectively.

Domain-routed multi-teacher OPD uses each training input's domain label to select a teacher~\citep{ma2026mopd}.
Indexing each domain by its corresponding teacher, let $d(x)\in\{1,\ldots,K\}$ denote the domain label of input $x$. 
The student is trained to minimize
\begin{equation}
\mathcal{L}_{\mathrm{MOPD}}(\theta)
=
\mathbb{E}_{x\sim\mathcal{D},\,y\sim\pi_\theta(\cdot\mid x)}
\left[
\frac{1}{|y|}
\sum_{t=1}^{|y|}
D_{\mathrm{KL}}
\left(
p_{\theta,t}\,\|\,q_{d(x),t}
\right)
\right].
\label{eq:mopd}
\end{equation}
Thus, all tokens in a response are distilled from the same domain-specific teacher, using reverse KL at student-generated prefixes~\citep{gu2024minillm_opd,openmopd2026}.

\section{TrustMOPD: Reliability-Aware Supervision Allocation}
\label{sec:method}

TrustMOPD computes and calibrates a reference-relative displacement score for each teacher at every student-generated state, then uses these scores as proxies for local reliability to allocate supervision.
The resulting weights determine each teacher's contribution to student training, enabling adaptive supervision without domain labels or a learned router.

\subsection{Reference-Relative Displacement as a Proxy for Local Reliability}
\label{sec:method-displacement}

We use domain specialists obtained by applying domain-specific RL to a shared initial policy $\pi_{\mathrm{ref}}$ as teachers, with $\pi_{\mathrm{ref}}$ serving as the reference for measuring their changes in next-token preferences.
Figure~\ref{fig:allocation-analysis}(a) shows that each teacher's mean displacement is largest on its own domain, supporting a connection between reference-relative policy change and specialization.
We therefore use displacement as the basis of a local reliability proxy, with cross-teacher calibration introduced in \S\ref{sec:method-weight}.
Unlike teacher--student divergence, which can reflect both expertise gaps and stylistic differences, teacher--reference displacement measures the changes induced by specialist training.

Let $\mathcal{V}$ denote the vocabulary and $q_{\mathrm{ref},t}=\pi_{\mathrm{ref}}(\cdot\mid s_t)$ the reference's next-token distribution at state $s_t$.
We define the centered reference-relative displacement $d_{k,t}\in\mathbb{R}^{|\mathcal{V}|}$ and its magnitude $\rho_{k,t}$ by
\begin{equation}
\begin{aligned}
d_{k,t}(v)
&:= \log\frac{q_{k,t}(v)}{q_{\mathrm{ref},t}(v)}
-\frac{1}{|\mathcal{V}|}
\sum_{v'\in\mathcal{V}}
\log\frac{q_{k,t}(v')}{q_{\mathrm{ref},t}(v')},\quad 
\rho_{k,t}:= \lVert d_{k,t}\rVert_2 .
\end{aligned}
\label{eq:method-displacement}
\end{equation}
The mean is taken over the vocabulary at the same state.
Centering removes the token-independent component of the log-policy ratio, which carries no information about relative next-token preferences.
The Euclidean norm then discards the direction of the signed displacement while retaining its magnitude, yielding one nonnegative score per teacher.
Because both distributions are conditioned on the student's current prefix, $\rho_{k,t}$ can vary across positions within a response.
Further analysis in \S\ref{sec:properties} provides an exact log-odds interpretation independent of the teacher's training objective and a utility-contrast interpretation under an idealized KL-regularized optimum.

\subsection{From Displacement to Supervision Weights}
\label{sec:method-weight}

Although $\rho_{k,t}$ measures how strongly teacher $k$ departs from the reference at state $s_t$, its absolute scale is not directly comparable across teachers. 
Independent RL runs may differ in reward scale, regularization strength, and update magnitude, leading to systematic differences in displacement scale across teachers.
For example, Figure~\ref{fig:allocation-analysis}(a) reports mean raw displacements of $6.33$ for the code teacher and $4.08$ for the domain-matched mathematics teacher on mathematics prompts.
As illustrated in Figure~\ref{fig:framework}(b), we normalize each displacement by the teacher's token-weighted mean magnitude on a fixed calibration distribution.
Let $\pi_{\theta_0}$ denote the student before distillation, $\mathcal{D}_{\mathrm{cal}}$ an unlabeled calibration sample, and $\mathbb{E}_{\mathrm{cal}}$ the expectation over $x\sim\mathcal{D}_{\mathrm{cal}}$ and $y\sim\pi_{\theta_0}(\cdot\mid x)$. We define
\begin{wrapfigure}{r}{0.4\linewidth}
\centering
\includegraphics[width=\linewidth]{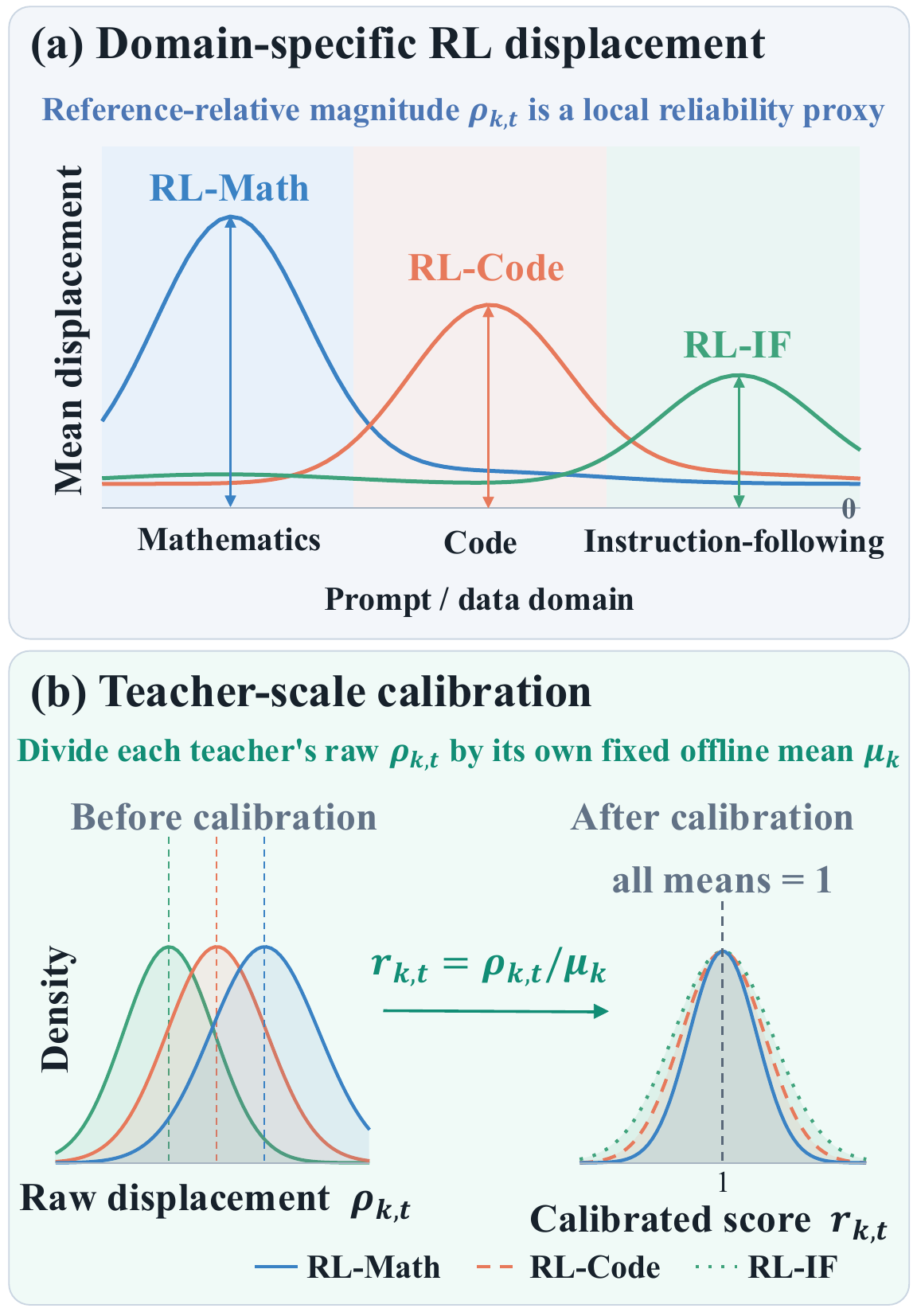}
\caption{\textbf{(a)} Schematic of domain-dependent teacher-reference displacement, with different scales across teachers.
\textbf{(b)} Dividing each teacher's displacement by its fixed calibration mean $\mu_k$ gives each teacher's scores a token-weighted mean of one on the calibration distribution.}
\vspace{-4mm}
\label{fig:framework}
\end{wrapfigure}
\begin{equation}
\mu_k
:=
\frac{
\mathbb{E}_{\mathrm{cal}}
\left[\sum_{t=1}^{|y|}\rho_{k,t}\right]
}{
\mathbb{E}_{\mathrm{cal}}\left[|y|\right]
},
\qquad
r_{k,t}:=\frac{\rho_{k,t}}{\mu_k}.
\label{eq:method-scale}
\end{equation}
The calibration constants $\{\mu_k\}_{k=1}^{K}$ are estimated once and fixed throughout distillation.
Each teacher therefore has unit token-weighted mean score on the calibration distribution, with $r_{k,t}>1$ indicating above-baseline displacement.
This normalization cancels any teacher-specific multiplicative factor shared by the calibration and current states.
Cross-teacher comparison is relative to the chosen calibration distribution and requires no domain labels.

TrustMOPD converts the calibrated scores into supervision weights through power normalization:
\begin{equation}
w_{k,t}
:=
\frac{r_{k,t}^{\,\gamma}}
{\sum_{j=1}^{K}r_{j,t}^{\,\gamma}},
\qquad
\sum_{k=1}^{K}w_{k,t}=1,
\label{eq:method-weight}
\end{equation}
where $\gamma>0$ controls how strongly differences in calibrated scores affect supervision allocation.
We use $\gamma>1$ to amplify relative differences in calibrated scores and concentrate supervision on higher-scoring teachers.
Larger values of $\gamma$ produce more concentrated allocations.

\subsection{Reliability-Weighted Multi-Teacher Distillation}
\label{sec:method-final}
We extend \eqref{eq:mopd} to a weighted sum of teacher-specific distillation losses at each student-generated state:
\begin{equation}
\mathcal{L}_{\mathrm{MOPD}}(\theta;w)
:=
\mathbb{E}_{\substack{x\sim\mathcal{D}\\
y\sim\pi_\theta(\cdot\mid x)}}
\left[
\frac{1}{|y|}
\sum_{t=1}^{|y|}
\sum_{k=1}^{K}
w_{k,t}\,
D_{\mathrm{KL}}
\left(
p_{\theta,t}\,\|\,q_{k,t}
\right)
\right],
\qquad
w_{k,t}\geq 0,\quad
\sum_{k=1}^{K}w_{k,t}=1.
\label{eq:weighted-mopd}
\end{equation}
Setting $w_{k,t}=\mathbb{I}\{k=d(x)\}$ at every position recovers the domain-routed MOPD objective in \eqref{eq:mopd}.
TrustMOPD instead replaces these one-hot weights with the reliability weights defined in Section~\ref{sec:method-weight}:
\begin{equation}
\underbrace{
w_{k,t}^{\mathrm{MOPD}}
=
\mathbb{I}\{k=d(x)\}
}_{\text{label-based}}
\quad\longrightarrow\quad
\underbrace{
w_{k,t}^{\mathrm{Trust}}
=
\frac{\left(\rho_{k,t}/\mu_k\right)^\gamma}
{\sum_{j=1}^{K}\left(\rho_{j,t}/\mu_j\right)^\gamma}
}_{\text{label-free}}.
\label{eq:method-weight-replacement}
\end{equation}
Substituting $w^{\mathrm{Trust}}$ into \eqref{eq:weighted-mopd} defines $\mathcal{L}_{\mathrm{TrustMOPD}}(\theta)$.
The weights adapt to each student-generated state while keeping the total teacher weight at one.
Their computation requires no domain labels or a learned router.
The weights are computed on each sampled rollout and treated as constants during the corresponding update via stop-gradient.

\subsection{Properties of the TrustMOPD Allocation}
\label{sec:properties}

For positive teacher and reference distributions, $d_{k,t}(v)-d_{k,t}(u)$ exactly equals the teacher--reference change in log odds between tokens $v$ and $u$, independently of the training objective.
The following results give a utility interpretation; proofs and further discussion appear in Appendix~\ref{apd:theory}.

\begin{theorem}[Utility-contrast recovery]
\label{thm:utility-contrast}
Assume each teacher $\pi_k$ is Bellman-optimal at all calibration and distillation states for a finite-horizon problem with finite rewards and a per-state KL penalty relative to the full-support reference $\pi_{\mathrm{ref}}$, with fixed coefficient $\beta_k>0$.
Let $Q_k^\star(s_t,v)$ denote the current token reward plus optimal KL-regularized future value,
$\bar Q_{k,t}^\star$ its vocabulary mean, and $\mathbf{1}$ the all-ones vector.
Then
\begin{equation}
d_{k,t}(v)
=
\frac{Q_k^\star(s_t,v)-\bar Q_{k,t}^\star}{\beta_k},
\qquad
\rho_{k,t}
=
\frac{
\lVert Q_k^\star(s_t,\cdot)-\bar Q_{k,t}^\star\mathbf{1}\rVert_2
}{\beta_k}.
\label{eq:rl-utility-contrast}
\end{equation}
\end{theorem}
Theorem~\ref{thm:utility-contrast} therefore interprets $\rho_{k,t}$ as measuring how strongly continuation utility varies across next-token choices for a given teacher.

\begin{corollary}[Relative-utility allocation]
\label{cor:relative-utility}
Under the assumptions of Theorem~\ref{thm:utility-contrast}, let
$\sigma_{k,t}:=\lVert Q_k^\star(s_t,\cdot)-\bar Q_{k,t}^\star\mathbf{1}\rVert_2$
and let $\bar\sigma_k$ be its token-weighted calibration mean, defined as in \eqref{eq:method-scale}.
If $\bar\sigma_k>0$ for every teacher, then at any state where $\sum_j\sigma_{j,t}>0$,
\begin{equation}
r_{k,t}
=
\frac{\sigma_{k,t}}{\bar\sigma_k},
\qquad
w_{k,t}^{\mathrm{Trust}}
=
\frac{(\sigma_{k,t}/\bar\sigma_k)^\gamma}
{\sum_{j=1}^{K}(\sigma_{j,t}/\bar\sigma_j)^\gamma}.
\label{eq:relative-utility-allocation}
\end{equation}
\end{corollary}
Thus, calibration cancels the explicit $1/\beta_k$ factor, and supervision follows utility-contrast magnitudes relative to each teacher's calibration baseline.

\section{Experiments}
\label{sec:experiments}

\subsection{Setup}
\label{sec:exp-setup}

\paragraph{Models.}
By default, we use the SmolLM3-3B checkpoints released by OpenMOPD~\citep{bakouch2025smollm3,openmopd2026}.
The mixed-domain SFT checkpoint serves as both the student initialization $\pi_{\theta_0}$ and the shared reference $\pi_{\mathrm{ref}}$.
The $K=3$ teachers are domain-specific RL fine-tunes of this checkpoint for mathematics, code, and instruction following.
Appendix~\ref{apd:exp-models} lists all model checkpoints and links.

\paragraph{Training data.}
We construct two training sets of $7{,}458$ prompts each, excluding evaluation data from both.
\textsc{SingleCap}, sampled from OpenMOPD~\citep{openmopd2026}, contains $1{,}540$ mathematics, $1{,}870$ code, and $4{,}048$ instruction-following prompts, each with a unique domain label.
Inspired by MathIF~\citep{fu-etal-2026-mathif}, we construct \textsc{MultiCap} with both single-capability prompts and prompts combining two capabilities: mathematics with instruction-following constraints, code with instruction-following constraints, and mathematics--code compositions.
For these mixed-capability prompts, no single domain label fully captures the expertise required for teacher selection.
We compare methods only within each training set.
Appendix~\ref{apd:exp-data} details data construction and filtering.

\paragraph{Baselines.}
All distillation methods share the same training protocol and differ only in teacher supervision allocation.
We compare TrustMOPD with \textsc{Uniform} weighting ($w_{k,t}=1/K$), \textsc{Random} weighting sampled independently at each token from the uniform distribution over the teacher simplex, and three single-teacher variants that use the mathematics, code, or instruction-following teacher throughout training.
The Dirichlet distribution is uniform over nonnegative teacher-weight vectors that sum to one.
On \textsc{SingleCap}, we additionally evaluate \textsc{MOPD}~\citep{ma2026mopd}, which selects one teacher per response using the ground-truth domain label.
We omit this baseline on \textsc{MultiCap}, where mixed-capability prompts lack a unique domain label.
We also report three training-free references: the initial student, the arithmetic parameter average of the three teachers~\citep{wortsman2022soups}, and \textsc{Routed Teachers}, which directly uses the corresponding RL teacher to generate responses for each evaluation domain.
Baseline details appear in Appendix~\ref{apd:exp-baselines}.

\paragraph{Evaluation.}
We evaluate mathematics on AIME25~\citep{maa2025aime} and AIME26~\citep{maa2026aime}, code on LiveCodeBench v5/v6~\citep{jain2024livecodebench}, and instruction following on IFEval~\citep{zhou2023IFEval} and IFBench~\citep{pyatkin2025IFBench}.
For each example, we sample $16$ responses for mathematics, $5$ for code, and $1$ for instruction following, all at a temperature of $0.6$. Scores are averaged over the sampled responses.
Each domain score is the mean of its two benchmark scores; \emph{overall} is the unweighted mean across domains.
For trained methods, we evaluate the final checkpoint from each of five training seeds and report the mean $\pm$ standard deviation across seeds.
No evaluation set is used for checkpoint selection.
Additional experimental details, hyperparameters, algorithm pseudocode, and computing cost appear in Appendix~\ref{apd:exp-training}--\ref{apd:exp-algorithms},~\ref{apd:cost}.

\subsection{Overall Effectiveness}
\label{sec:exp-main}

\newcommand{\nostd}{\phantom{\std{0.00}}}

\begin{table}[t]
\centering
\footnotesize
\setlength{\tabcolsep}{5pt}
\renewcommand{\arraystretch}{0.82}
\caption{
For trained methods, results are mean $\pm$ standard deviation over five seeds, using the final checkpoint from each run.
\textbf{Bold} and \underline{underlined} values mark the best and second-best label-free methods within each block; \colorbox{labelgreen}{green} rows use ground-truth domain labels and are excluded from ranking.
$\Delta_{\mathrm{Init}}$ is the overall-score change from the initial student.
The ``teacher alone'' rows evaluate each specialist directly; their domain-matched scores form the \textsc{Routed Teachers} row.
}
\label{tab:main}

\begin{tabular}{@{}c@{\hspace{6pt}}lccccc}
\toprule
\multirow{2.4}{*}{Data}
& \multirow{2.4}{*}{Method / allocation rule}
& \multicolumn{1}{c}{Math}
& \multicolumn{1}{c}{Code}
& \multicolumn{1}{c}{IF}
& \multirow{2.4}{*}{Overall}
& \multirow{2.4}{*}{$\Delta_{Init}$} \\
\cmidrule(lr){3-3}
\cmidrule(lr){4-4}
\cmidrule(lr){5-5}
&
& {\small AIME 25/26}
& {\small LCB v5/v6}
& {\small IFEval/IFBench}
& & \\

\midrule
\multicolumn{7}{l}{\emph{No-distillation references}} \\

& \ctrlc{Student initialization}
& \ctrlc{17.02\nostd}
& \ctrlc{15.97\nostd}
& \ctrlc{41.66\nostd}
& \ctrlc{24.89\nostd}
& \ctrlc{$\phantom{+}0.00$} \\

& Math teacher alone \quad $\pi_{\mathrm{math}}$
& \labelc{22.50\nostd}
& 17.26\nostd
& 41.16\nostd
& 26.97\nostd
& $+2.08$ \\

& Code teacher alone \quad $\pi_{\mathrm{code}}$
& 20.62\nostd
& \labelc{21.34\nostd}
& 41.55\nostd
& 27.84\nostd
& $+2.95$ \\

& IF teacher alone \quad $\pi_{\mathrm{IF}}$
& 15.83\nostd
& 14.92\nostd
& \labelc{47.73\nostd}
& 26.16\nostd
& $+1.27$ \\

& \labelc{Oracle-routed teachers}
& \labelc{22.50\nostd}
& \labelc{21.34\nostd}
& \labelc{47.73\nostd}
& \labelc{30.52\nostd}
& \labelc{$+5.63$} \\

& Parameter-averaged teacher
& 19.29\nostd
& 16.84\nostd
& 43.86\nostd
& 26.66\nostd
& $+1.77$ \\

\midrule
\multirow{7}{*}{\rotatebox[origin=c]{90}{\textsc{SingleCap}}}

& \labelc{MOPD\tablefootnote{Our reported MOPD results use the OpenMOPD~\citep{openmopd2026} implementation.} \quad $\mathbf{e}_{z(x)}$}
& \labelc{22.58\std{0.77}}
& \labelc{20.48\std{0.39}}
& \labelc{47.49\std{1.37}}
& \labelc{30.18\std{0.25}}
& \labelc{$+5.29$} \\
\cmidrule(l){2-7}

& Uniform \quad $\tfrac{1}{K}\mathbf{1}$
& 20.14\std{0.23}
& 18.67\std{0.38}
& 45.04\std{1.23}
& \underline{27.95}\std{0.59}
& $+3.06$ \\

& Random \, $\operatorname{Dirichlet}(1,\ldots,1)$
& 20.22\std{0.95}
& 18.31\std{0.59}
& 44.62\std{0.99}
& 27.72\std{0.72}
& $+2.83$ \\

& Single teacher \quad $\mathbf{e}_{\mathrm{math}}$
& \underline{21.80}\std{0.61}
& 17.01\std{0.45}
& 40.80\std{0.73}
& 26.54\std{0.29}
& $+1.65$ \\

& Single teacher \quad $\mathbf{e}_{\mathrm{code}}$
& 20.44\std{1.48}
& \textbf{20.77}\std{0.56}
& 42.17\std{0.71}
& 27.79\std{0.55}
& $+2.90$ \\

& Single teacher \quad $\mathbf{e}_{\mathrm{IF}}$
& 16.41\std{0.50}
& 15.60\std{0.51}
& \textbf{48.33}\std{0.53}
& 26.78\std{0.13}
& $+1.89$ \\

& \oursc{TrustMOPD \quad $\propto \mathbf{r}_t^{\gamma}$}
& \oursc{\textbf{21.91}\std{0.61}}
& \oursc{\underline{20.67}\std{0.75}}
& \oursc{\underline{47.56}\std{0.48}}
& \oursc{\textbf{30.04}\std{0.47}}
& \oursc{$+5.15$} \\

\midrule
\multirow{6}{*}{\rotatebox[origin=c]{90}{\textsc{MultiCap}}}

& Uniform \quad $\tfrac{1}{K}\mathbf{1}$
& 20.20\std{0.43}
& 18.87\std{0.52}
& 43.88\std{1.42}
& 27.65\std{0.53}
& $+2.76$ \\

& Random \, $\operatorname{Dirichlet}(1,\ldots,1)$
& 20.24\std{1.13}
& 18.95\std{0.66}
& 44.69\std{1.20}
& \underline{27.96}\std{0.49}
& $+3.07$ \\

& Single teacher \quad $\mathbf{e}_{\mathrm{math}}$
& \underline{21.88}\std{0.66}
& 17.20\std{0.70}
& 41.80\std{1.28}
& 26.96\std{0.55}
& $+2.07$ \\

& Single teacher \quad $\mathbf{e}_{\mathrm{code}}$
& 20.18\std{0.94}
& \textbf{21.25}\std{0.37}
& 41.83\std{1.71}
& 27.75\std{0.69}
& $+2.86$ \\

& Single teacher \quad $\mathbf{e}_{\mathrm{IF}}$
& 15.52\std{0.87}
& 14.70\std{0.80}
& \underline{46.98}\std{0.80}
& 25.73\std{0.38}
& $+0.84$ \\

& \oursc{TrustMOPD \quad $\propto \mathbf{r}_t^{\gamma}$}
& \oursc{\textbf{22.92}\std{1.40}}
& \oursc{\underline{20.69}\std{0.50}}
& \oursc{\textbf{47.62}\std{1.41}}
& \oursc{\textbf{30.41}\std{0.60}}
& \oursc{$+5.52$} \\

\bottomrule
\end{tabular}
\vspace{-3mm}
\end{table}

Table~\ref{tab:main} evaluates two settings: \textsc{SingleCap} tests whether TrustMOPD can approach label-based routing without using domain labels, while \textsc{MultiCap} tests its effectiveness when training prompts combine capabilities.
The ``teacher alone'' rows show that no specialist's score on any other evaluated domain falls more than $1.3$ points below the initial student's, indicating limited cross-domain degradation in this setup.

\noindent\(\bullet\)\enspace
\textbf{Obs.\ 1: TrustMOPD approaches label-based routing on \textsc{SingleCap} without domain labels.}
TrustMOPD achieves $30.04\pm0.47$ overall, $0.14$ points below label-based MOPD and $2.09$ points above the strongest label-free baseline, \textsc{Uniform}.
This closes $91.5\%$ of the overall-score gap between the initial student and \textsc{Routed Teachers}, compared with $54.4\%$ for \textsc{Uniform}.
TrustMOPD improves scores in all three domains over the initial student.

\noindent\(\bullet\)\enspace
\textbf{Obs.\ 2: TrustMOPD remains effective with mixed-capability training data.}
On \textsc{MultiCap}, TrustMOPD reaches $30.41\pm0.60$ overall, exceeding the strongest label-free baseline, \textsc{Random}, by $2.45$ points.
This closes $98.0\%$ of the overall-score gap between the initial student and \textsc{Routed Teachers}, compared with $54.5\%$ for \textsc{Random}.
TrustMOPD outperforms the initial student across all domains and achieves the highest label-free scores in mathematics and instruction following.

\subsection{Analyzing State-Dependent Supervision Allocation}
\label{sec:exp-allocation}

\refstepcounter{footnote}
\edef\mathcodefootnotenumber{\number\value{footnote}}

\begin{figure}[t]
    \centering
    \includegraphics[width=\linewidth]{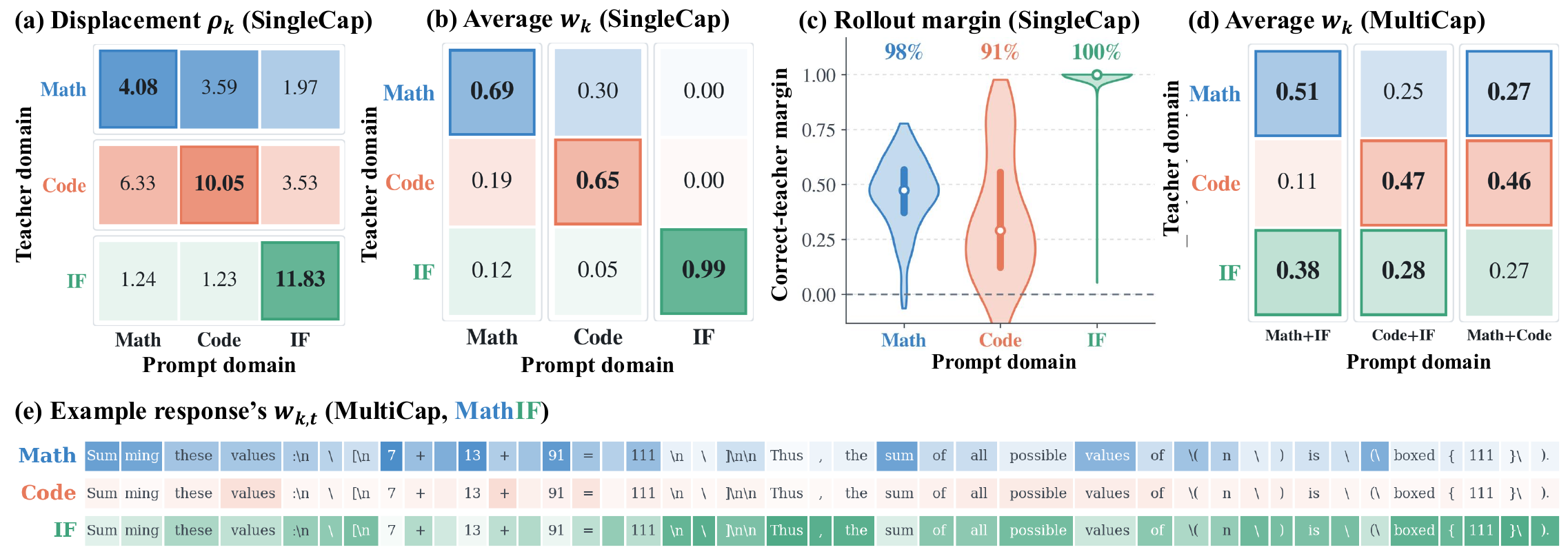}
    \caption{\textbf{State-dependent supervision allocation by TrustMOPD.}
    (a) Mean raw displacement $\rho_{k,t}$ from the shared reference. 
    Each teacher's mean displacement is largest on its matching domain, but raw magnitudes differ in scale across teachers.
    (b) Mean rollout-level teacher weights on \textsc{SingleCap}, grouped by prompt domain.
    (c) Distribution of the domain-matched teacher's margin $m_i=\bar{w}_{i,d_i}-\max_{k\neq d_i}\bar{w}_{i,k}$, where $\bar{w}_{i,k}$ is teacher $k$'s mean weight over tokens in rollout $i$, and $d_i=d(x_i)$ is the prompt's ground-truth domain index. Percentages report the fraction of rollouts with $m_i>0$.
    (d) Mean rollout-level teacher weights on \textsc{MultiCap}, grouped by capability composition; outlined cells mark relevant teachers.
    (e) Token-level teacher weights for an example mathematics-and-instruction-following response.
    Rows correspond to teachers; darker shading indicates higher weights.}
    \label{fig:allocation-analysis}
\end{figure}
We examine domain alignment in raw displacement, calibrated scores, and supervision weights, and illustrate how allocation varies within individual responses.
For rollout $i$ with response $y_i$, we compute teacher $k$'s mean weight as $\bar{w}_{i,k}=|y_i|^{-1}\sum_{t=1}^{|y_i|}w_{i,k,t}$ and group rollouts by domain or capability composition.
Labels are used only for analysis, never for allocation.

\noindent\(\bullet\)\enspace
\textbf{Obs.\ 3: Displacement and allocation exhibit domain alignment.}
We first analyze \textsc{SingleCap}, where each prompt targets a single capability and has a unique domain label.
Figure~\ref{fig:allocation-analysis}(a) shows that each teacher's mean raw displacement is highest on its own domain, although the scales differ across teachers.
Table~\ref{tab:apd_rmatrix} shows that after calibration, the domain-matched teacher has the highest mean score in each prompt domain.
Figure~\ref{fig:allocation-analysis}(b) shows corresponding alignment in mean teacher weights: the matched teachers receive $0.69$, $0.65$, and $0.99$ for mathematics, code, and instruction following, respectively.
Figure~\ref{fig:allocation-analysis}(c) further shows that the domain-matched teacher receives the largest rollout-averaged weight on $98\%$, $91\%$, and $100\%$ of rollouts, respectively.

\noindent\(\bullet\)\enspace
\textbf{Obs.\ 4: Allocation favors relevant teachers on average and varies within a response.}
Figure~\ref{fig:allocation-analysis}(d) shows that on \textsc{MultiCap}, relevant teachers jointly receive $0.89$, $0.75$, and $0.73$ of total weight for mathematics plus instruction following, code plus instruction following, and mathematics plus code, respectively, all exceeding the $2/3$ share under uniform weighting.
For mathematics-and-instruction-following prompts, the respective teacher weights are $0.51$ and $0.38$, compared with $0.11$ for the code teacher.
In the example response shown in Figure~\ref{fig:allocation-analysis}(e), the mathematics teacher receives its highest weights around the operands ``7'', ``13'', and ``91'' and their arithmetic composition.
After the result ``111'' is obtained, the instruction-following teacher dominates the concluding explanation and the formatting command ``\texttt{\textbackslash boxed\{111\}}'', while the code teacher remains weak throughout.
Additional teacher-weight plots and token-level visualizations covering all three capability compositions appear in Appendix~\ref{apd:alloc}, with further training-time allocation examples in Appendix~\ref{apd:case}.

\subsection{Ablations, Sensitivity, and Generalization}
\label{sec:exp-ablation}

\paragraph{Ablation settings.}
Figure~\ref{fig:component-ablation} compares ablations of three design choices.
First, we replace the reference-relative displacement proxy $\rho_{k,t}$ with teacher likelihood, negative teacher entropy, or teacher--student KL, calibrating each alternative before applying the same power normalization.
Second, ``Without $\mu_k$'' retains the displacement proxy but disables calibration by setting $\mu_k=1$.
Third, ``Response-level'' applies each teacher's rollout-averaged weight at every token, preserving its total weight within the rollout while removing within-response variation in allocation.
All variants share the student, teacher pool, and training protocol of full TrustMOPD.
Detailed definitions and formulas for all variants appear in Appendix~\ref{apd:ablation}.

\noindent\(\bullet\)\enspace
\textbf{Obs.\ 5: Displacement performs best among tested proxies, while calibration and token-level allocation contribute more on \textsc{MultiCap}.}
Relative to full TrustMOPD, alternative proxies reduce overall by $1.63$--$2.08$ points on \textsc{SingleCap} and $1.48$--$3.06$ on \textsc{MultiCap}.
Calibration and allocation granularity have larger effects on \textsc{MultiCap}: removing $\mu_k$ reduces overall by $1.28$ points, compared with $0.31$ on \textsc{SingleCap}; response-level allocation incurs a $0.92$-point drop, compared with $0.05$ on \textsc{SingleCap}.
These differences suggest that calibrating teacher scores and adapting weights within a response are more beneficial when training prompts combine multiple capabilities.
Table~\ref{tab:apd_ablation} provides absolute scores and per-domain results.

\begin{wrapfigure}{r}{0.4\textwidth}
    \centering
    \includegraphics[width=\linewidth]{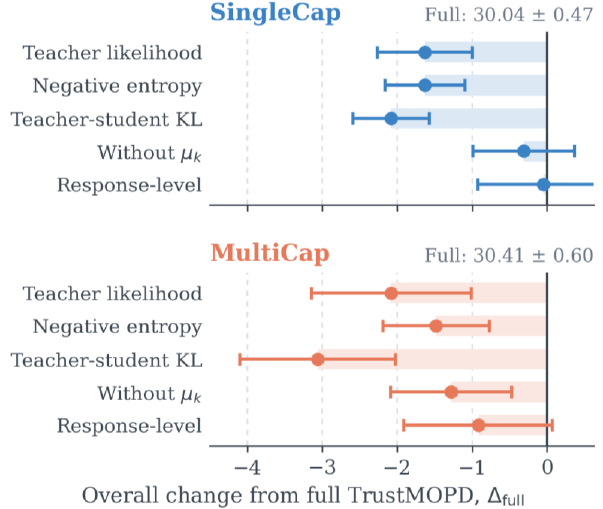}
    \vspace{-4mm}
    \caption{\textbf{Core component ablations on \textcolor{singlecapblue}{\textsc{SingleCap}} and \textcolor{multicaporange}{\textsc{MultiCap}}.} Each row changes one component of TrustMOPD while holding other settings fixed. 
    Points show mean overall-score differences between each variant and full TrustMOPD; the vertical line at zero marks the reference. Error bars denote standard deviations of paired score differences across five seeds. Further right is better.}
    \vspace{-6mm}
    \label{fig:component-ablation}
\end{wrapfigure}

\paragraph{Sensitivity settings.}
We vary the sharpness exponent $\gamma\in\{3,5,7,9\}$ with other settings fixed and include \textsc{Uniform} as a reference.
We report overall gains over the initial student.
Unless explicitly varied, all sensitivity and transfer experiments use $\gamma=7$ and the default training hyperparameters.
For both analyses, we evaluate the final checkpoint from each of five training seeds.

\noindent\(\bullet\)\enspace
\textbf{Obs.\ 6: Performance remains stable on \textsc{SingleCap} but peaks at $\gamma=7$ on \textsc{MultiCap}.}
Figure~\ref{fig:robustness-generalization}(a) shows that TrustMOPD outperforms matched \textsc{Uniform} controls at every tested $\gamma$.
On \textsc{SingleCap}, gains remain between $5.11$ and $5.40$ points as the mean maximum teacher weight increases from $0.66$ to $0.86$ in panel~(b).
On \textsc{MultiCap}, gains rise from $4.42$ at $\gamma=3$ to $5.52$ at $\gamma=7$, then decline to $4.87$ at $\gamma=9$ despite further concentration.
These results suggest that sharpening allocation benefits mixed-capability training up to a point, but further concentration need not improve performance.
Table~\ref{tab:apd_gamma} provides absolute scores.

\paragraph{Transfer settings.}
For backbone transfer, we replace the default SmolLM3-3B setup with DeepSeek-R1-Distill-Qwen-7B~\citep{deepseek2025r1} as the student initialization and shared reference, together with three public RL-finetuned teachers derived from this checkpoint.
For data-source transfer, we rebuild both training sets using DeepMath-103K~\citep{he2025deepmath}, KodCode-V1~\citep{xu2025kodcode}, and the T\"ulu 3 persona instruction-following dataset~\citep{lambert2024tulu3}, matching the original dataset sizes and capability compositions while retaining the default student and teachers.
Evaluation data are excluded from both training sets.
We report overall gains over \textsc{Uniform} under the same model and data settings.
Model checkpoints and data construction details appear in Appendices~\ref{apd:exp-models} and~\ref{apd:exp-data}, respectively.

\noindent\(\bullet\)\enspace
\textbf{Obs.\ 7: TrustMOPD remains effective across model backbones and training-data sources.}
Figure~\ref{fig:robustness-generalization}(c) shows that on DeepSeek-R1-Distill-Qwen-7B, TrustMOPD exceeds matched \textsc{Uniform} controls in overall score by $1.40$ points on \textsc{SingleCap} and $1.15$ on \textsc{MultiCap}.
Figure~\ref{fig:robustness-generalization}(d) shows that TrustMOPD remains effective with independently sourced prompts, outperforming matched \textsc{Uniform} controls by $2.22$ points on \textsc{SingleCap} and $1.95$ on \textsc{MultiCap}.
Its overall scores are slightly lower than those obtained with the corresponding default training sets, by $0.60$ and $0.79$ points, respectively.
These results support the transferability of TrustMOPD's advantage over uniform weighting.
Table~\ref{tab:apd_transfer} provides absolute scores.

\begin{figure}[t]
\centering
\includegraphics[width=\linewidth]{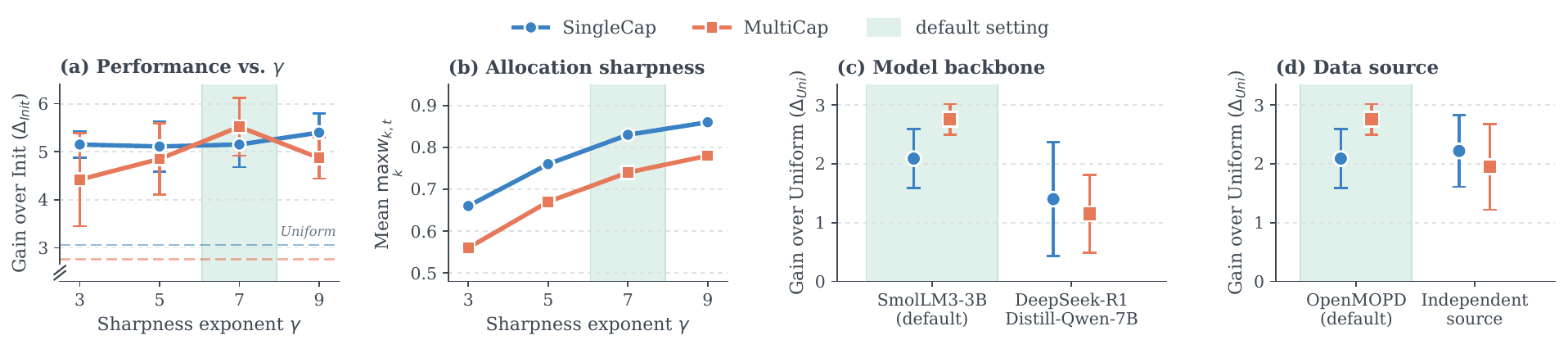}
\caption{\textbf{Sensitivity and transfer of TrustMOPD.}
(a) Overall gain over the initial student, $\Delta_{\mathrm{Init}}$, as the sharpness exponent $\gamma$ varies. Dashed lines indicate the corresponding \textsc{Uniform} baselines.
(b) Allocation sharpness, measured by $\mathbb{E}_{s_t}[\max_k w_{k,t}]$, as $\gamma$ varies.
(c) Transfer to a different backbone and teacher pool.
(d) Transfer across training-data sources with the models fixed.
In (c,d), $\Delta_{\mathrm{Uni}}$ is the overall gain over the matched \textsc{Uniform} baseline, computed by subtracting scores within each seed.
Results in (a,c,d) are means over five seeds; error bars denote standard deviations of these gains.
The $y$-axis in (a) is truncated.}
\label{fig:robustness-generalization}
\end{figure}

\section{Related Work}
\label{sec:related_work}

\paragraph{On-policy distillation and adaptive supervision.}
OPD queries teachers at student-generated prefixes to provide supervision on the states encountered during student training~\citep{agarwal2024onpolicy_opsd,gu2024minillm_opd,lu2025onpolicydistillation}; \citet{zan2026echoesoracle} review the broader literature.
Studies of teacher--student compatibility and OPD failures examine when such supervision is useful~\citep{li2026rethinkingopd,fu2026revisitingopd,wang2026demystifyingopd}.
This motivates evaluating supervision locally instead of assuming that a teacher is equally useful throughout a response.
Adaptive methods address this issue at different granularities.
At the prompt or trajectory level, they select, order, or gate examples using correctness, student alignment, verifier feedback, and reliability estimates~\citep{zhang2026brts,akhondzadeh2026rgopd,zhang2026tgopd,zhu2026reorder}.
At finer granularity, methods select informative or teachable tokens and reasoning prefixes~\citep{huang2025selectkd,xu2026tip_opd,wang2026teachability,zhang2026fastopd}.
For multi-turn agents, SOD adjusts step-level distillation strength using teacher--student divergence~\citep{zhong2026sod}.
Other methods adjust updates according to token position or predicted entropy effects~\citep{pwopsd2026,yang2026influencedirected}, or combine trajectory filtering with token reweighting~\citep{li2026fireopd}.
Rollout construction offers another way to control the states receiving supervision: Guided-OPD gradually reduces teacher-generated turns during rollout collection~\citep{li2026guidedopd}.
These approaches establish supervision selection and weighting as important dimensions of OPD.
TrustMOPD studies the related allocation problem that arises when several complementary teachers are available at the same student-generated state.

\paragraph{Multi-teacher distillation and supervision allocation.}
Multi-teacher KD combines teacher outputs or representations~\citep{you2017multiteacher,fukuda2017ensemble}, with adaptive variants learning instance-level weights or sampling teachers during training~\citep{yuan2021rlkd,ding2024dynakd}.
Under offline supervision, EWAD uses teacher entropy and inter-teacher agreement to balance teacher and gold supervision~\citep{sumit2026reliabilitygated}.
In OPD, MOPD assigns a domain-matched teacher to each response, while Open-MOPD balances optimization budgets and refreshes rewards under oracle routing~\citep{ma2026mopd,openmopd2026}.
Recent methods introduce further adaptation through domain sampling, answer-based teacher verification, or post-debate confidence weighting~\citep{sun2026d3mopd,he2026mtsdpo,wang2026madopd}.
Uncertainty-calibrated MOPD combines prompt-level routing with trajectory filtering and entropy-calibrated token-update gating~\citep{liu2026uncertaintymopd}.
Closer to token-level teacher allocation, H-OPD weights vision-language and text-only teachers using prediction entropy, while VG-OPD localizes verified expert gains through teacher--student disagreement~\citep{yin2026hopd,xu2026vgopd}.
These approaches differ in allocation granularity and the evidence used to assess teacher usefulness.
However, domain-level expertise and confidence measures need not identify the most useful teacher at every intermediate generation state.
TrustMOPD derives token-level teacher weights from calibrated displacement relative to a shared reference.
This provides a common basis for comparing specialist teachers at each prefix, without domain labels, outcome verifiers, learned routers, or teacher debate.
\section{Conclusion}
\label{sec:conclusion}

We introduce TrustMOPD, a label-free method that allocates token-level supervision using calibrated teacher displacement from a shared reference.
Across both training corpora, it achieves the highest overall score among the compared label-free distillation methods and approaches label-based MOPD on \textsc{SingleCap}.
Analyses show domain alignment in displacement and supervision weights and illustrate within-response changes in allocation.
Transfer experiments further show gains over matched uniform-weighting baselines on an additional backbone and independently sourced training data.
These results support calibrated reference-relative displacement as a useful signal for integrating complementary teacher capabilities.
The current formulation assumes a shared reference.
Future work includes extending TrustMOPD to teachers without a common initialization and evaluating it on larger models and a broader range of datasets and domains.

\bibliography{references}
\bibliographystyle{references}

\appendix

\makeatletter
\@ifundefined{lemma}{\newtheorem{lemma}{Lemma}}{}
\@ifundefined{proposition}{\newtheorem{proposition}{Proposition}}{}
\@ifundefined{assumption}{\newtheorem{assumption}{Assumption}}{}
\@ifundefined{proof}{%
  \newenvironment{proof}{\par\noindent\textit{Proof.}\ }%
  {\hfill\textit{Q.E.D.}\par}%
}{}
\makeatother

\section{Theoretical Analysis and Structural Properties}
\label{apd:theory}

This appendix develops the utility interpretation and structural properties of TrustMOPD.
\S\ref{apd:theory-setting} defines the sequential RL setting, and \S\ref{apd:displacement} proves the utility-contrast and calibration results.
\S\ref{apd:proofs} explains centering and the frozen calibration divisor, while \S\ref{apd:proof-mixture} derives the mixed log-target implementation.
\S\ref{apd:disp-scope} connects these results to practical teacher distributions.

\subsection{Sequential RL setting and optimal-teacher assumption}
\label{apd:theory-setting}

We model autoregressive generation as a finite-horizon Markov decision process.
At step $t$, the state $s_t=(x,y_{<t})$ consists of a prompt $x$ and the generated prefix $y_{<t}$, and the action $v_t\in\mathcal{V}$ is the next token.
The transition appends the selected token to the prefix until termination.
For teacher $k$, let $r_k(s_t,v_t)$ denote the reward contribution associated with its specialization.
This formulation includes terminal-only rewards as a special case.

Let $\pi_{\mathrm{ref}}$ be a fixed reference policy with full support over the token vocabulary at every nonterminal state, and let $\beta_k>0$ be a teacher-specific regularization coefficient.
We consider the KL-regularized objective
\begin{equation}
J_k(\pi)
=
\mathbb{E}_{\tau\sim\pi}
\left[
\sum_{t=1}^{|y|}r_k(s_t,v_t)
-
\beta_k\sum_{t=1}^{|y|}
\log\frac{\pi(v_t\mid s_t)}{\pi_{\mathrm{ref}}(v_t\mid s_t)}
\right],
\label{eq:apd-sequential-objective}
\end{equation}
where $\tau$ denotes a generation trajectory and $|y|$ is its response length.
The regularization term is equivalently $\beta_k$ times the expected sum of
$D_{\mathrm{KL}}\!\left(\pi(\cdot\mid s_t)\,\|\,\pi_{\mathrm{ref}}(\cdot\mid s_t)\right)$
over states visited by $\pi$.

The utility interpretation in \S\ref{apd:displacement} uses the following assumption.
\begin{assumption}[Optimal teacher]
\label{asm:converged}
For every teacher $k$, $\pi_k$ coincides with the Bellman-optimal policy $\pi_k^\star$ for the objective in equation~\ref{eq:apd-sequential-objective} at every state considered in the analysis, including calibration states and states visited by the student.
\end{assumption}

Let $V_k^\star(s)$ denote the optimal KL-regularized return from state $s$, with $V_k^\star(s)=0$ at terminal states.
Define the optimal continuation utility as
\begin{equation}
Q_k^\star(s,v)
=
r_k(s,v)
+
\mathbb{E}_{s'\mid s,v}\bigl[V_k^\star(s')\bigr].
\label{eq:apd-soft-q}
\end{equation}
Thus, $Q_k^\star(s,v)$ combines the immediate reward with the optimal regularized return after taking action $v$.

\subsection{Proofs of Theorem~\ref{thm:utility-contrast} and Corollary~\ref{cor:relative-utility}}
\label{apd:displacement}

We first derive the closed-form optimal policy for the sequential objective in \eqref{eq:apd-sequential-objective}.
Under Assumption~\ref{asm:converged}, taking its log-ratio relative to the reference and centering over the vocabulary yields Theorem~\ref{thm:utility-contrast}.
Normalizing the resulting displacement magnitude by its calibration mean then gives Corollary~\ref{cor:relative-utility}.

\paragraph{Closed-form optimal policy.}
At every nonterminal state $s$, Bellman optimality gives
\begin{equation}
V_k^\star(s)
=
\max_{q\in\Delta(\mathcal{V})}
\left\{
\sum_{v\in\mathcal{V}}q(v)Q_k^\star(s,v)
-
\beta_k\,\mathrm{KL}
\bigl(q\,\|\,\pi_{\mathrm{ref}}(\cdot\mid s)\bigr)
\right\},
\label{eq:apd-bellman-optimality}
\end{equation}
where $\Delta(\mathcal{V})$ is the probability simplex over the token vocabulary.
The continuation utility $Q_k^\star$ incorporates the immediate reward and all subsequent rewards and KL penalties through $V_k^\star$.
Thus, the state-wise optimization in \eqref{eq:apd-bellman-optimality} accounts for the full remaining generation trajectory.

\begin{lemma}[Reference-regularized optimal policy]
\label{lem:apd-bellman-policy}
The maximizer of \eqref{eq:apd-bellman-optimality} is unique and satisfies
\begin{equation}
\pi_k^\star(v\mid s)
=
\pi_{\mathrm{ref}}(v\mid s)
\exp\left(
\frac{Q_k^\star(s,v)-V_k^\star(s)}{\beta_k}
\right),
\label{eq:apd-bellman-policy}
\end{equation}
where
\begin{equation}
V_k^\star(s)
=
\beta_k\log
\sum_{v'\in\mathcal{V}}
\pi_{\mathrm{ref}}(v'\mid s)
\exp\left(\frac{Q_k^\star(s,v')}{\beta_k}\right).
\label{eq:apd-soft-value}
\end{equation}
\end{lemma}

\begin{proof}
Fix $s$ and abbreviate
$Q(v)=Q_k^\star(s,v)$,
$q_{\mathrm{ref}}(v)=\pi_{\mathrm{ref}}(v\mid s)$,
and $\beta=\beta_k$.
Define the normalized distribution
$\widehat q(v)=q_{\mathrm{ref}}(v)\exp(Q(v)/\beta-c_s)$,
where
\begin{equation}
c_s
=
\log\sum_{v'}q_{\mathrm{ref}}(v')
\exp\left(\frac{Q(v')}{\beta}\right),
\qquad
\log\frac{\widehat q(v)}{q_{\mathrm{ref}}(v)}
=
\frac{Q(v)}{\beta}-c_s.
\label{eq:apd-stationarity}
\end{equation}
Because $q_{\mathrm{ref}}$ has full support, so does $\widehat q$.
For any $q\in\Delta(\mathcal{V})$, substitution gives
\[
\sum_v q(v)Q(v)
-
\beta\,\mathrm{KL}(q\,\|\,q_{\mathrm{ref}})
=
\beta c_s
-
\beta\,\mathrm{KL}(q\,\|\,\widehat q).
\]
Since $\beta>0$ and KL divergence is nonnegative, the objective is uniquely maximized at $q=\widehat q$, with optimal value $\beta c_s$.
Consequently, $V_k^\star(s)=\beta c_s$, which gives \eqref{eq:apd-soft-value}.
Substituting this equality into the definition of $\widehat q$ gives \eqref{eq:apd-bellman-policy}.
\end{proof}

\paragraph{Proof of Theorem~\ref{thm:utility-contrast}.}
Fix a student-generated state $s_t$.
Under Assumption~\ref{asm:converged}, write
$q_{k,t}=\pi_k(\cdot\mid s_t)=\pi_k^\star(\cdot\mid s_t)$
and
$q_{\mathrm{ref},t}=\pi_{\mathrm{ref}}(\cdot\mid s_t)$.
Taking logarithms in \eqref{eq:apd-bellman-policy} yields
\begin{equation}
\log\frac{q_{k,t}(v)}{q_{\mathrm{ref},t}(v)}
=
\frac{Q_k^\star(s_t,v)-V_k^\star(s_t)}{\beta_k}.
\label{eq:apd-converged-logratio}
\end{equation}
Define the uniform vocabulary mean
$\bar Q_{k,t}^\star
=
|\mathcal{V}|^{-1}\sum_{v'}Q_k^\star(s_t,v')$.
Averaging \eqref{eq:apd-converged-logratio} over the vocabulary gives
$(\bar Q_{k,t}^\star-V_k^\star(s_t))/\beta_k$.
Subtracting this mean, as prescribed by \eqref{eq:method-displacement}, cancels the state-value term and gives
\begin{equation}
d_{k,t}(v)
=
\frac{Q_k^\star(s_t,v)-\bar Q_{k,t}^\star}{\beta_k},
\qquad
\rho_{k,t}
=
\frac{\left\|Q_k^\star(s_t,\cdot)-\bar Q_{k,t}^\star\mathbf{1}\right\|_2}{\beta_k}.
\label{eq:apd-utility-contrast}
\end{equation}
This establishes \eqref{eq:rl-utility-contrast} and proves Theorem~\ref{thm:utility-contrast}.
The same derivation applies to a selected candidate set $\mathcal{A}_t$ by taking the centering mean and displacement norm over $\mathcal{A}_t$.
Renormalizing the teacher and reference distributions on this set contributes only token-independent offsets, which centering removes.
For this version, the calibration argument below uses the corresponding candidate-restricted utility contrasts and the same candidate-selection rule used to compute displacement.

Centering removes the common utility offset while preserving scaled differences between actions.
\S\ref{apd:proofs} establishes the corresponding distribution-level invariance without requiring teacher optimality.

\paragraph{Proof of Corollary~\ref{cor:relative-utility}.}
Let
\begin{equation}
\sigma_{k,t}
:=
\left\|Q_k^\star(s_t,\cdot)-\bar Q_{k,t}^\star\mathbf{1}\right\|_2,
\qquad
\bar\sigma_k
:=
\frac{
\mathbb{E}_{\mathrm{cal}}
\left[\sum_{t=1}^{|y|}\sigma_{k,t}\right]
}{
\mathbb{E}_{\mathrm{cal}}[|y|]
}.
\label{eq:apd-sigma-calibration}
\end{equation}
Here, $\bar\sigma_k$ is the token-weighted mean utility-contrast magnitude under the calibration distribution.
Assume $\bar\sigma_k>0$ for every teacher.
Because $\beta_k$ is constant across states for teacher $k$, substituting \eqref{eq:apd-utility-contrast} into \eqref{eq:method-scale} gives
\begin{equation}
\mu_k
=
\frac{\bar\sigma_k}{\beta_k},
\qquad
r_{k,t}
=
\frac{\rho_{k,t}}{\mu_k}
=
\frac{\sigma_{k,t}}{\bar\sigma_k}.
\label{eq:apd-relative-utility}
\end{equation}
At any state where $\sum_j\sigma_{j,t}>0$, at least one calibrated score is positive, so the normalization denominator in \eqref{eq:method-weight} is nonzero.
Since $z\mapsto z^\gamma$ is strictly increasing on $[0,\infty)$ for $\gamma>0$, exponentiation and normalization preserve the ranking of the calibrated scores:
\begin{equation}
\arg\max_k w^{\mathrm{Trust}}_{k,t}
=
\arg\max_k r_{k,t}
=
\arg\max_k\frac{\sigma_{k,t}}{\bar\sigma_k}.
\label{eq:apd-relative-ranking}
\end{equation}
Equations~\eqref{eq:apd-relative-utility}--\eqref{eq:apd-relative-ranking} establish \eqref{eq:relative-utility-allocation} and prove Corollary~\ref{cor:relative-utility}.

Calibration therefore cancels the explicit $1/\beta_k$ factor and expresses each teacher's current utility-contrast magnitude relative to its own calibration mean.
TrustMOPD assigns the greatest supervision weight to the teacher with the largest relative utility contrast at the current student-generated state.

\subsection{Centering and the frozen calibration divisor}
\label{apd:proofs}

This subsection explains two design choices in TrustMOPD: centering the teacher-reference log-ratio and normalizing its magnitude by a frozen calibration divisor.
The resulting properties follow from the construction and do not require Assumption~\ref{asm:converged}.

\paragraph{Centering on a candidate set.}
Fix a state and a candidate set $\mathcal{A}$ with $|\mathcal{A}|=n\ge2$.
Let $q_k$ and $q_{\mathrm{ref}}$ be the teacher and reference distributions, both positive on $\mathcal{A}$.
Define the restricted log-ratio vector and its centered version as
\[
g_k(v)=\log\frac{q_k(v)}{q_{\mathrm{ref}}(v)},
\qquad
\bar g_k=\frac{1}{n}\sum_{v\in\mathcal{A}}g_k(v),
\qquad
d_k=g_k-\bar g_k\mathbf{1}.
\]
For a vector $g$ on $\mathcal{A}$, the corresponding reference-weighted distribution is
\[
p_g(v)
=
\frac{q_{\mathrm{ref}}(v)\exp(g(v))}
{\sum_{u\in\mathcal{A}}q_{\mathrm{ref}}(u)\exp(g(u))}.
\]
In particular, $p_{g_k}$ equals the teacher distribution renormalized on $\mathcal{A}$.

\begin{proposition}[Canonical centered representative]
\label{prop:gauge}
Adding a constant to every coordinate of $g_k$ leaves $p_{g_k}$ unchanged.
Among all vectors $g_k+\lambda\mathbf{1}$ with $\lambda\in\mathbb{R}$, $d_k$ is the unique minimum-norm representative and the orthogonal projection of $g_k$ onto $\mathbf{1}^{\perp}$.
Moreover,
\begin{equation}
\lVert g_k\rVert_2^2
=
\lVert d_k\rVert_2^2+n\bar g_k^2.
\label{eq:apd-split}
\end{equation}
\end{proposition}

\begin{proof}
A common shift by $\lambda$ multiplies the numerator and denominator of $p_{g_k}$ by the same factor $\exp(\lambda)$, leaving the distribution unchanged.
The squared norm of the shifted vector is
\begin{equation}
\lVert g_k+\lambda\mathbf{1}\rVert_2^2
=
\lVert g_k\rVert_2^2
+2\lambda\langle g_k,\mathbf{1}\rangle
+n\lambda^2.
\label{eq:apd-shift-norm}
\end{equation}
This strictly convex quadratic is uniquely minimized at $\lambda=-\bar g_k$, yielding $d_k$.
Since $\langle d_k,\mathbf{1}\rangle=0$, the decomposition
$g_k=d_k+\bar g_k\mathbf{1}$
is orthogonal, establishing the projection property and \eqref{eq:apd-split}.
\end{proof}

Centering also preserves every pairwise change in log odds:
\begin{equation}
d_k(v)-d_k(v')
=
\log\frac{q_k(v)}{q_k(v')}
-
\log\frac{q_{\mathrm{ref}}(v)}{q_{\mathrm{ref}}(v')},
\qquad v,v'\in\mathcal{A}.
\label{eq:apd-logodds-contrast}
\end{equation}
Thus, the centered vector retains how the teacher changes the relative preference between candidate tokens, while its norm excludes the common-offset component in \eqref{eq:apd-split}.
Renormalizing either distribution on $\mathcal{A}$ changes its log-probabilities only by a constant and therefore leaves $d_k$ unchanged.
These identities apply exactly on the selected candidate set, over which the displacement norm is computed.

\paragraph{A fixed baseline across training batches.}
TrustMOPD uses a frozen calibration divisor $\mu_k>0$ while recomputing the displacement magnitude $\rho_{k,t}$ on the current student's states and candidate sets.

To see the distinction from batch-local normalization, fix teacher $k$ and a batch $B$ containing $m\ge1$ token positions.
Let
$\hat\mu_k(B)=m^{-1}\sum_{t\in B}\rho_{k,t}$
be the current batch mean.
When $\hat\mu_k(B)>0$, normalizing by this mean gives
\begin{equation}
\frac{1}{m}\sum_{t\in B}
\frac{\rho_{k,t}}{\hat\mu_k(B)}
=1.
\label{eq:apd-batchlocal-mean}
\end{equation}
Batch-local normalization therefore fixes each teacher's mean normalized score at $1$, regardless of changes in its mean displacement across batches.
With the frozen divisor, the corresponding mean is instead
\begin{equation}
\frac{1}{m}\sum_{t\in B}r_{k,t}
=
\frac{1}{m}\sum_{t\in B}\frac{\rho_{k,t}}{\mu_k}
=
\frac{\hat\mu_k(B)}{\mu_k}.
\label{eq:apd-frozen-mean}
\end{equation}
A value above or below $1$ indicates that the current batch's mean displacement is respectively above or below the teacher's initial calibration mean.
The frozen divisor thus preserves changes in mean displacement relative to a fixed baseline, while the numerator adapts to the states visited by the evolving student.

\subsection{Exact aggregation into a mixed log-target}
\label{apd:proof-mixture}

The weighted distillation objective in \eqref{eq:weighted-mopd} can be evaluated using a single mixed log-target.
This follows from the affine dependence of reverse KL on teacher log-probabilities and does not require Assumption~\ref{asm:converged}.
We work at a fixed student-generated state and suppress the token index.
All distributions below are normalized on the same support $\mathcal{A}$, which is the selected candidate set when candidate-based distillation is used.
Teacher probabilities are positive on this support, and $\ell_k=\log q_k$ denotes teacher $k$'s log-probability vector.

\paragraph{Aggregation of affine losses.}
\begin{proposition}[Exact aggregation for affine token losses]
\label{prop:mixture}
Suppose the token loss has the form
\begin{equation}
\ell_{\mathrm{distill}}(q_k,p_\theta)
=
\langle a(p_\theta),\ell_k\rangle+b(p_\theta),
\label{eq:apd-affine-loss}
\end{equation}
where $a(p_\theta)$ and $b(p_\theta)$ do not depend on teacher $k$.
For weights $w_k\ge0$ satisfying $\sum_{k=1}^{K}w_k=1$, define
$\ell_{\mathrm{mix}}=\sum_{k=1}^{K}w_k\ell_k$.
Then
\begin{equation}
\sum_{k=1}^{K}w_k\ell_{\mathrm{distill}}(q_k,p_\theta)
=
\langle a(p_\theta),\ell_{\mathrm{mix}}\rangle+b(p_\theta).
\label{eq:apd-affine-aggregation}
\end{equation}
\end{proposition}

\begin{proof}
Expanding the weighted sum gives
\[
\sum_{k=1}^{K}w_k
\bigl(\langle a(p_\theta),\ell_k\rangle+b(p_\theta)\bigr)
=
\left\langle a(p_\theta),\sum_{k=1}^{K}w_k\ell_k\right\rangle
+
b(p_\theta)\sum_{k=1}^{K}w_k.
\]
Substituting the definition of $\ell_{\mathrm{mix}}$ and
$\sum_{k=1}^{K}w_k=1$ yields \eqref{eq:apd-affine-aggregation}.
\end{proof}

\paragraph{Reverse KL and the geometric-mean target.}
Reverse KL satisfies the proposition because
\begin{equation}
D_{\mathrm{KL}}(p_\theta\Vert q_k)
=
\langle p_\theta,\log p_\theta\rangle
-
\langle p_\theta,\ell_k\rangle.
\label{eq:apd-reverse-kl-affine}
\end{equation}
Consequently, its weighted sum is exactly
\[
\sum_{k=1}^{K}w_kD_{\mathrm{KL}}(p_\theta\Vert q_k)
=
\langle p_\theta,\log p_\theta\rangle
-
\langle p_\theta,\ell_{\mathrm{mix}}\rangle.
\]
The mixed vector $\ell_{\mathrm{mix}}$ is the logarithm of the unnormalized weighted geometric mean
$\prod_{k=1}^{K}q_k(v)^{w_k}$.
Its normalized distribution is
\begin{equation}
\widetilde q_{\mathrm{mix}}(v)
=
\frac{\exp(\ell_{\mathrm{mix}}(v))}{Z_{\mathrm{mix}}},
\qquad
Z_{\mathrm{mix}}
=
\sum_{v\in\mathcal{A}}\exp(\ell_{\mathrm{mix}}(v)).
\label{eq:apd-normalized-mixture}
\end{equation}
Since
$\log\widetilde q_{\mathrm{mix}}(v)
=
\ell_{\mathrm{mix}}(v)-\log Z_{\mathrm{mix}}$,
we obtain
\begin{equation}
\sum_{k=1}^{K}w_kD_{\mathrm{KL}}(p_\theta\Vert q_k)
=
D_{\mathrm{KL}}(p_\theta\Vert\widetilde q_{\mathrm{mix}})
-
\log Z_{\mathrm{mix}}.
\label{eq:apd-mixture-kl}
\end{equation}

\paragraph{Implementation consequence.}
Within each distillation update, the sampled states, selected support, teacher scores, and allocation weights are held fixed for differentiation.
Using $\ell_{\mathrm{mix}}$ directly in the affine expression preserves the scalar loss exactly.
Our implementation normalizes the mixed log-target and evaluates a single reverse-KL loss
$D_{\mathrm{KL}}(p_\theta\Vert\widetilde q_{\mathrm{mix}})$.
By \eqref{eq:apd-mixture-kl}, this loss has the same student gradient as the original weighted objective because $Z_{\mathrm{mix}}$ is constant with respect to the student parameters during the update.

\subsection{Scope of the theoretical interpretation}
\label{apd:disp-scope}

\paragraph{Utility interpretation.}
Under Assumption~\ref{asm:converged}, \S\ref{apd:displacement} establishes an exact connection between centered teacher-reference displacement and utility contrast.
The derivation considers a fixed full-support reference and a KL-regularized sequential objective with a positive teacher-specific coefficient $\beta_k$ that is constant across states.
This setting provides a theoretical interpretation of the displacement signal and its calibration.

\paragraph{Applicability beyond KL-regularized training.}
TrustMOPD operates directly on the observed teacher and reference distributions and does not require an explicit KL penalty during teacher training.
Even for teachers trained without KL regularization, centered displacement preserves the teacher-reference changes in pairwise log odds in \eqref{eq:apd-logodds-contrast}.
The centering property in Proposition~\ref{prop:gauge}, the frozen-divisor analysis in \S\ref{apd:proofs}, and the aggregation identity in Proposition~\ref{prop:mixture} likewise hold under their stated support and normalization conditions.
These distribution-level properties provide the basis for applying TrustMOPD across teacher-training configurations, while Assumption~\ref{asm:converged} supplies the additional utility interpretation.

\section{Experimental Details}
\label{apd:exp}

This appendix documents the experimental setup.
\S\ref{apd:exp-models} lists the model checkpoints, and \S\ref{apd:exp-data} describes training-data construction and decontamination, including the settings used for backbone and data-source transfer, respectively.
\S\ref{apd:exp-baselines} defines the baselines and reference methods.
\S\ref{apd:exp-training} describes the training and calibration protocol, \S\ref{apd:exp-hparams} lists the hyperparameters, and \S\ref{apd:exp-algorithms} provides the calibration and training algorithms.

\subsection{Models}
\label{apd:exp-models}

\begin{table}[!t]
\centering
\footnotesize
\setlength{\tabcolsep}{4pt}
\caption{
\textbf{Model checkpoints used in our experiments.}
Each block contains a shared pre-RL checkpoint and three distinct RL-finetuned teachers used for mathematics, code, and instruction following, respectively.
The \colorbox{ctrlgray}{shaded} row serves as both the student initialization $\pi_{\theta_0}$ and the reference $\pi_{\mathrm{ref}}$ for computing teacher displacement.
The default setting uses the Open-MOPD SmolLM3-3B checkpoints.
The backbone transfer setting uses three publicly released teachers sharing DeepSeek-R1-Distill-Qwen-7B as their pre-RL initialization.
All identifiers refer to Hugging Face repositories.
}
\vspace{1mm}
\label{tab:apd_checkpoints}
\begin{tabular}{@{}lll@{}}
\toprule
Role & HuggingFace identifier & Reference \\
\midrule
\multicolumn{3}{@{}p{\textwidth}@{}}{\emph{Default setting: SmolLM3-3B (Table~\ref{tab:main}, and blocks (a,b,d) of Figure~\ref{fig:robustness-generalization})}} \\
\ctrlc{$\pi_{\theta_0}$ and $\pi_{\mathrm{ref}}$}
& \ctrlc{\texttt{BytedTsinghua-SIA/Open-MOPD-SmolLM3-3B-MixSFT}}
& \ctrlc{\citet{bakouch2025smollm3}} \\
Teacher, math & \texttt{BytedTsinghua-SIA/Open-MOPD-SmolLM3-3B-RL-Math} & \citet{openmopd2026} \\
Teacher, code & \texttt{BytedTsinghua-SIA/Open-MOPD-SmolLM3-3B-RL-Code} & \citet{openmopd2026} \\
Teacher, IF & \texttt{BytedTsinghua-SIA/Open-MOPD-SmolLM3-3B-RL-IF} & \citet{openmopd2026} \\
\midrule
\multicolumn{3}{@{}p{\textwidth}@{}}{\emph{Backbone transfer: DeepSeek-R1-Distill-Qwen-7B (block (c) of Figure~\ref{fig:robustness-generalization}; teachers not trained by us)}} \\
\ctrlc{$\pi_{\theta_0}$ and $\pi_{\mathrm{ref}}$}
& \ctrlc{\texttt{deepseek-ai/DeepSeek-R1-Distill-Qwen-7B}}
& \ctrlc{\citet{deepseek2025r1}} \\
Teacher, math & \texttt{nvidia/AceMath-RL-Nemotron-7B} & \citet{chen2025acemathrl} \\
Teacher, code & \texttt{nvidia/AceReason-Nemotron-7B} & \citet{chen2025acereason} \\
Teacher, IF & \texttt{THU-KEG/R1-Distill-Qwen-7B-VerIF} & \citet{peng2025verif} \\
\bottomrule
\end{tabular}

\vspace{1mm}
\parbox{\textwidth}{\footnotesize Licenses differ across rows and none of them is unrestricted: the Open-MOPD models and the VerIF teacher are Apache-2.0, DeepSeek-R1-Distill-Qwen-7B is MIT, and the two NVIDIA teachers are released under the NVIDIA Open Model License.}
\end{table}

Table~\ref{tab:apd_checkpoints} lists the model repositories and checkpoints used in our experiments.
Each model group consists of a shared pre-RL reference and three publicly released teachers obtained by RL fine-tuning from that reference.
This relationship allows \eqref{eq:method-displacement} to measure each teacher's distributional change from its pre-RL initialization.

\paragraph{Default model setup.}
Our default experiments use the SmolLM3-3B checkpoints released by Open-MOPD.
The mixed-domain SFT checkpoint serves as both the student initialization $\pi_{\theta_0}$ and the shared reference $\pi_{\mathrm{ref}}$.
The former gives all compared distillation methods the same starting weights; the latter provides the common baseline for measuring teacher displacement.
Using the same checkpoint for both roles is an experimental choice, as the allocation rule does not require the student to be initialized from the reference.

The $K=3$ teachers specialize in mathematics, code, and instruction following, respectively, and were trained using reinforcement learning with verifiable rewards.
Their released checkpoints correspond to different training durations: a few hundred steps for the mathematics and code teachers and several thousand steps for the instruction-following teacher.
The instruction-following teacher also uses a shorter response-length limit.
These differences in training configuration motivate measuring each teacher's displacement relative to its own calibration baseline through \eqref{eq:method-scale}.
\S\ref{apd:scales} reports the resulting calibration constants.

\paragraph{Backbone transfer setup.}
For the backbone transfer experiment in Figure~\ref{fig:robustness-generalization}(c), we use DeepSeek-R1-Distill-Qwen-7B~\citep{deepseek2025r1} as both the student initialization and the shared reference.
We use three distinct publicly released teacher checkpoints: AceMath-RL-Nemotron-7B as the mathematics teacher~\citep{chen2025acemathrl}, AceReason-Nemotron-7B as the code teacher~\citep{chen2025acereason}, and R1-Distill-Qwen-7B-VerIF as the instruction-following teacher~\citep{peng2025verif}.
All three were obtained by RL fine-tuning from DeepSeek-R1-Distill-Qwen-7B, allowing teacher displacement to be computed relative to their shared pre-RL initialization.

We compute a separate calibration constant for each teacher using the same procedure as in the default setup.
This experiment replaces the student, reference, and all three teacher checkpoints while retaining the TrustMOPD allocation procedure.
Table~\ref{tab:apd_checkpoints} provides the exact repository identifiers.

\subsection{Training data}
\label{apd:exp-data}

The two default training sets, \textsc{SingleCap} and \textsc{MultiCap}, each contain $7458$ prompts drawn from the same pool of source data.
For each set, we follow the corresponding Open-MOPD data composition, preserving its per-category sample counts and proportions~\citep{openmopd2026}.

\paragraph{Single-capability training set.}
\textsc{SingleCap} contains $1540$ mathematics, $1870$ code, and $4048$ instruction-following prompts.
Each prompt is assigned to one of these three domains, providing the domain label $d(x)$ used by the MOPD routing rule in \eqref{eq:mopd}.

\paragraph{Mixed-capability training set.}
\textsc{MultiCap} contains six prompt categories:
$256$ mathematics,
$256$ code,
$1579$ instruction-following,
$2131$ mathematics--instruction-following,
$2980$ code--instruction-following,
and $256$ mathematics--code prompts.
The three mixed-capability categories account for $5367$ prompts ($72.0\%$), while the remaining $2091$ prompts ($28.0\%$) belong to a single domain.

For a mixed-capability prompt, the composition identifies two relevant specialist domains but does not by itself specify which single teacher should supervise the entire response.
Applying a single-teacher routing rule to these prompts would therefore require an additional assignment rule.

\paragraph{Comparison protocol.}
Matching the total prompt counts keeps the training-set size fixed across the two settings.
Within each setting, the compared distillation methods use the same training prompts.
The two settings address complementary questions: \textsc{SingleCap} measures the effectiveness of allocation without domain labels against a label-based routing reference, while \textsc{MultiCap} evaluates allocation when individual prompts combine capabilities.

\subsubsection{Corpus for data-source transfer}
\label{apd:exp-data-external}

\begin{table}[!ht]
\centering
\footnotesize
\setlength{\tabcolsep}{4pt}
\caption{
\textbf{Sources and composition of the data-source transfer corpus.}
The two training sets are used in Figure~\ref{fig:robustness-generalization}(d).
Each column matches the category counts and proportions of its corresponding default training set, following the Open-MOPD composition.
Mixed-capability prompts combine a source problem with an instruction-following (IF) constraint block or a mathematical derivation block.
The mathematics--code category uses problems from KodCode-V1.
Construction details appear in \S\ref{apd:exp-data-external}.
}
\vspace{1mm}
\label{tab:apd_datasets}

\begin{tabular}{@{}
>{\raggedright\arraybackslash}p{0.235\textwidth}
>{\raggedright\arraybackslash}p{0.45\textwidth}
rr@{}}
\toprule
Prompt category
& Source and construction
& \textsc{SingleCap}
& \textsc{MultiCap} \\
\midrule

Mathematics
& \texttt{zwhe99/DeepMath-103K}~\citep{he2025deepmath}
& $1540$ & $256$ \\

Code
& \texttt{KodCode/KodCode-V1}~\citep{xu2025kodcode}
& $1870$ & $256$ \\

Instruction following
& \texttt{allenai/tulu-3-sft-personas-\allowbreak instruction-following}~\citep{lambert2024tulu3}
& $4048$ & $1579$ \\

\cmidrule{1-4}

Mathematics--IF
& DeepMath-103K $+$ IF constraint block
& --- & $2131$ \\

Code--IF
& KodCode-V1 $+$ IF constraint block
& --- & $2980$ \\

Mathematics--code
& KodCode-V1 $+$ mathematical derivation block
& --- & $256$ \\

\midrule
\multicolumn{2}{@{}l}{Total}
& $7458$ & $7458$ \\
\bottomrule
\end{tabular}
\end{table}

For the data-source transfer experiment in Figure~\ref{fig:robustness-generalization}(d), we replace the training prompts while retaining the default student initialization, reference, and teacher checkpoints.
Table~\ref{tab:apd_datasets} lists the replacement data sources.
The source datasets are released under MIT (DeepMath-103K), CC-BY-NC-4.0 (KodCode-V1), and ODC-BY (T\"ulu~3 persona instruction following).
All KodCode-based categories use problems from the \texttt{Algorithm} and \texttt{Data\_Structure} subsets of KodCode-V1.
We construct separate \textsc{SingleCap} and \textsc{MultiCap} training sets, each containing $7458$ prompts and matching the per-category counts of its default counterpart.
The overlap-screening procedure is described in \S\ref{apd:exp-decontam}.

\paragraph{Mathematics and code with instruction-following constraints.}
For the mathematics--instruction-following and code--instruction-following categories, we pair each replacement problem with an instruction-following constraint block from a corresponding default \textsc{MultiCap} prompt.
We replace the original problem statement and copy the constraint block verbatim, preserving its natural-language description and associated checker identifiers and arguments.
This construction preserves the number of occurrences of each checker identifier exactly.
It therefore changes the source problems while retaining the instruction-following constraints attached to them.

\paragraph{Mathematics--code composition.}
We construct mathematics--code prompts from the replacement code pool by appending the same fixed instruction block requesting a mathematical derivation.
Each prompt therefore combines an explicit programming task with a derivation requirement, while the underlying code problem varies across examples.

\paragraph{Matching criteria.}
In addition to per-category sample counts, we match the proportions of standalone instruction-following prompts containing one, two, or three constraints and the prompt-length buckets of the corresponding default sets.
Difficulty, mathematical topic, and code-task format are determined by the replacement sources and are not explicitly matched.

\paragraph{Use in distillation.}
The reconstructed sets supply prompts for student rollouts, with supervision provided by the fixed teachers through \eqref{eq:weighted-mopd}.
The T\"ulu rule checkers and KodCode unit tests are not used to compute training rewards.

\subsubsection{Decontamination}
\label{apd:exp-decontam}

\begin{table}[!t]
\centering
\footnotesize
\setlength{\tabcolsep}{4.5pt}
\caption{
\textbf{Decontamination statistics for the data-source transfer corpus.}
The four removal columns report within-pool duplicates and matches
against the comparison corpora using the rules in
\S\ref{apd:exp-decontam}.
All three candidate pools are screened against the same
$113{,}742$ comparison texts, covering evaluation datasets and
the default teachers' RL training prompts.
\emph{Exact} denotes normalized exact matching;
\emph{5-gram} denotes word-level $5$-gram Jaccard similarity
of at least $0.8$.
\emph{Kept} reports the candidate pool remaining before
quota-based sampling of the training sets in
Table~\ref{tab:apd_datasets}.
}
\vspace{1mm}
\label{tab:apd_decontam}

\begin{tabular}{@{}lrrrrrr@{}}
\toprule
& & \multicolumn{2}{c}{Within-pool}
  & \multicolumn{2}{c}{Cross-corpus} & \\
\cmidrule(lr){3-4}
\cmidrule(lr){5-6}
Pool & Candidates & Exact & 5-gram & Exact & 5-gram & Kept \\
\midrule

DeepMath-103K
& $103{,}022$
& $2{,}945$ & $1{,}010$
& $2{,}441$ & $194$
& $96{,}432$ \\

KodCode-V1
& $64{,}249$
& $25{,}718$ & $14$
& $0$ & $0$
& $38{,}517$ \\

T\"ulu~3 personas IF
& $29{,}980$
& $18$ & $0$
& $0$ & $0$
& $29{,}962$ \\

\bottomrule
\end{tabular}
\end{table}

We screen all upstream candidates for overlap with the comparison corpora and with previously accepted candidates before sampling the final training sets.
Table~\ref{tab:apd_decontam} reports the numbers removed by each rule and the sizes of the retained candidate pools.

\paragraph{Matching rules.}
A candidate is excluded if either of two criteria identifies a match.
The first is exact equality after Unicode NFKC normalization, case folding, and removal of non-alphanumeric characters.
The second is a word-level $5$-gram Jaccard similarity of at least $0.8$, computed exactly.
For code candidates, we additionally apply both criteria to the upstream parent instruction.
This step accounts for KodCode's construction of multiple examples from a shared seed problem and screens for overlap at the parent-instruction level as well as the rendered-prompt level.

\paragraph{Comparison corpora.}
The comparison pool contains $113{,}742$ texts.
It covers the six benchmarks in \S\ref{sec:exp-setup}, nine additional evaluation datasets spanning mathematics, code, and instruction following, and the $104{,}848$-prompt RL training pool used by the default teachers.

\paragraph{Screening and sampling order.}
A candidate is accepted only if it passes screening against both the comparison pool and the previously accepted candidates.
We then sample from the retained pools using the category and prompt-length quotas in \S\ref{apd:exp-data-external}.
The \emph{Kept} column in Table~\ref{tab:apd_decontam} reports these pools before sampling, rather than the final training sets.

\subsubsection{Example prompts}
\label{apd:exp-data-examples}

Figure~\ref{fig:apd-data-examples} illustrates three construction patterns in the data-source transfer corpus: an unchanged source prompt, a source problem paired with an instruction-following constraint, and a code problem paired with a mathematical derivation request.

\begin{figure}[!t]
\centering
\begin{tcolorbox}[
  colback=white, colframe=black!70, boxrule=0.6pt,
  title={Example prompts from the data-source transfer corpus},
  fonttitle=\bfseries\small,
  coltitle=white, colbacktitle=black!70,
  top=4pt, bottom=4pt, left=5pt, right=5pt
]
\small

\noindent\textbf{(a) \textsc{SingleCap}: instruction following}

\noindent{\scriptsize Source:
\texttt{allenai/tulu-3-sft-personas-instruction-following}}

\noindent\colorbox{gray!10}{\parbox{0.975\linewidth}{\footnotesize
Write a technical report on a new statistical method you have researched. Divide the report into three sections. In the first section, provide a detailed explanation of the theoretical foundation of the method. [\ldots] End the report with the exact sentence: ``This method bridges the gap between theory and practice.''
}}

\noindent{\scriptsize
$\triangleright$ The source prompt is used unchanged, with no appended instructions.
Its constraint annotations
(\texttt{specific ending}, \texttt{format:number of sections})
are retained as metadata.
}

\vspace{5pt}
\noindent\textbf{(b) \textsc{MultiCap}: mathematics--instruction following}

\noindent{\scriptsize Source:
\texttt{zwhe99/DeepMath-103K}}

\noindent\colorbox{gray!10}{\parbox{0.975\linewidth}{\footnotesize
On the unit sphere defined by \(x^2 + y^2 + z^2 = 1\), consider \(n\) points \(A_1, A_2, \ldots, A_n\) where \(n \geq 2\). Determine the maximum possible value of the sum \(\sum_{1 \leq i < j \leq n} |A_i A_j|^2\).

\vspace{3pt}
\textcolor{ifgreen}{\rule{\linewidth}{0.4pt}}
\vspace{1pt}

\textcolor{ifgreen}{There should be 2 paragraphs. Paragraphs are separated with the markdown divider: ***}
}}

\noindent{\scriptsize
$\triangleright$ The source mathematics problem is followed by an
\textcolor{ifgreen}{instruction-following constraint}
copied verbatim from a matched default \textsc{MultiCap} prompt.
The associated checker identifier
\texttt{paragraphs:paragraphs}
and arguments \texttt{\{\}} are retained as metadata.
The original problem from the matched prompt is replaced.
}

\vspace{5pt}
\noindent\textbf{(c) \textsc{MultiCap}: mathematics--code}

\noindent{\scriptsize Source:
\texttt{KodCode/KodCode-V1}}

\noindent\colorbox{gray!10}{\parbox{0.975\linewidth}{\footnotesize
\#\#\# Binary Search Tree (BST) Deletion --- Implement the delete operation in a Binary Search Tree and ensure the tree retains its properties after deletion. Write a function \texttt{delete\_node} that deletes a node with a given value from a BST. [\ldots]

\vspace{3pt}
\textcolor{mathteal}{\rule{\linewidth}{0.4pt}}
\vspace{1pt}

\textcolor{mathteal}{Alongside your implementation, derive the answer mathematically: state the closed-form expression or recurrence your algorithm is computing, justify it, and give its complexity. Put the final expression in \texttt{\textbackslash boxed\{\}}.}
}}

\noindent{\scriptsize
$\triangleright$ The source code problem is followed by a fixed
\textcolor{mathteal}{mathematical derivation request}.
The appended block is identical across all $256$ prompts in this category.
}

\end{tcolorbox}

\caption{
\textbf{Examples of the data construction described in
\S\ref{apd:exp-data-external}.}
The displayed text preserves the prompt wording, with omissions in problem statements marked by $[\ldots]$.
Colored separators and explanatory annotations are added for presentation and are not part of the prompts.
Code--instruction-following prompts follow the construction in (b), using a code problem in place of the mathematics problem.
Each example is the first prompt in its category in stored order.
Category sizes are reported in Table~\ref{tab:apd_datasets}.
}
\label{fig:apd-data-examples}
\end{figure}

\subsection{Baselines and references}
\label{apd:exp-baselines}

The trained allocation baselines follow the common protocol in
\S\ref{apd:exp-training}.
Within each experimental setting, they use the same student initialization and training prompts, with teacher supervision determined by the allocation rules below.

\paragraph{Fixed teacher weighting.}
\textsc{Uniform} assigns equal weight to all teachers,
$w_{k,t}=1/K$, at every token, providing a baseline for evaluating adaptive allocation.
The three single-teacher baselines each select one fixed specialist $k_0$,
with $w_{k,t}=\mathbb{I}\{k=k_0\}$ throughout training.
They perform on-policy distillation from the mathematics, code, or instruction-following teacher, respectively, measuring the performance obtained from each specialist individually.

\paragraph{Random token-level weighting.}
\textsc{Random} independently samples a weight vector at each token:
$w_{\cdot,t}\sim\operatorname{Dirichlet}(1,\ldots,1)$.
This distribution is uniform over nonnegative teacher-weight vectors that sum to one.
The weights vary across tokens but are independent of the student-generated state $s_t$.
This baseline tests whether non-uniform token-level weighting alone can account for the gains from adaptive allocation.

\paragraph{Label-based routing.}
\textsc{MOPD} uses the supplied domain label $d(x)$ to select one teacher for the entire response, as in \eqref{eq:mopd}:
$w_{k,t}=\mathbb{I}\{k=d(x)\}$ for every token $t$.
It serves as a label-based routing reference on \textsc{SingleCap}, where every prompt has one assigned domain.
We omit it on \textsc{MultiCap}, where $72.0\%$ of prompts combine two capabilities and lack a unique domain label for single-teacher routing.

\paragraph{References without distillation.}
We report three references that do not use the distillation training prompts and are therefore shared across \textsc{SingleCap} and \textsc{MultiCap} for a fixed model setup.
The initial student establishes the performance before distillation.
The parameter-averaged teacher uses the arithmetic mean of the three specialists' parameters,
$\theta_{\mathrm{avg}}=K^{-1}\sum_{k=1}^{K}\theta_k$,
providing a model-merging reference without additional training.
\textsc{Routed Teachers} uses the evaluation-domain label to select the corresponding specialist to generate each response.
It therefore retains three teacher models and requires domain labels at inference time.

\subsection{Training and calibration protocol}
\label{apd:exp-training}

\paragraph{Shared training settings.}
All distillation methods use the training settings of the underlying MOPD implementation, without method-specific hyperparameter tuning.
Each run consists of $116$ training steps with a batch size of $64$.
Prompts are truncated to at most $4096$ tokens, and the maximum response length is $8192$ tokens.

\paragraph{Candidate-token scoring.}
At each student-generated state, we select the student's top $C=128$ candidate tokens.
Teacher and reference probabilities are renormalized over this shared candidate set before computing the centered displacements in \S\ref{sec:method-displacement}.
The student distribution is likewise renormalized over the candidate set when forming the distillation objective.
The teacher log-probabilities are combined using the allocation weights, and the resulting mixed log-target is normalized over the candidate set to form the reverse-KL distillation target.
During each update, sampled states and candidate sets are held fixed, and teacher scores and allocation weights are detached.

\paragraph{Calibration.}
Before distillation training, we estimate the teacher-specific constants $\mu_k$ in \eqref{eq:method-scale} using eight batches of rollouts generated by the initial student.
Calibration uses the same candidate-selection rule and response-length limit as training.
The resulting constants are frozen throughout the training run, while displacement scores are recomputed on the evolving student's states.

We calibrate separately for each model setup and training prompt set.
This accounts for the dependence of the calibration scale on the teacher, prompt distribution, and response-length window; \S\ref{apd:scales} reports the corresponding constants.
Calibration uses no domain labels.
For variants using calibrated scores, we follow the same protocol and re-estimate the constants for each scoring rule.

\subsection{Hyperparameters}
\label{apd:exp-hparams}

\begin{table}[!t]
\centering
\footnotesize
\setlength{\tabcolsep}{5pt}
\caption{
\textbf{Training, calibration, and evaluation settings.}
\emph{Shared} denotes settings common to the compared distillation methods.
\emph{Calibrated} denotes settings used by TrustMOPD and the proxy ablations that consume calibrated scores in Table~\ref{tab:apd_ablation}.
The default allocation exponent is $\gamma=7$; its sensitivity analysis uses $\{3,5,7,9\}$.
Evaluation sampling settings are specified separately from training rollouts.
}
\vspace{1mm}
\label{tab:apd_hparams}

\begin{tabular}{@{}
>{\raggedright\arraybackslash}p{0.30\textwidth}
>{\raggedright\arraybackslash}p{0.35\textwidth}
>{\centering\arraybackslash}p{0.15\textwidth}
@{}}
\toprule
Hyperparameter & Value & Applies to \\
\midrule

\multicolumn{3}{@{}l}{\emph{Optimization}} \\
Optimizer
& AdamW ($\beta_1=0.9$, $\beta_2=0.999$)
& \ctrlc{Shared} \\
Peak learning rate
& $1\times10^{-6}$
& \ctrlc{Shared} \\
LR schedule
& Constant after linear warmup
& \ctrlc{Shared} \\
Warmup ratio
& $0.03$
& \ctrlc{Shared} \\
Weight decay
& $0.0$
& \ctrlc{Shared} \\
Gradient clipping
& $1.0$
& \ctrlc{Shared} \\
Training steps
& $116$
& \ctrlc{Shared} \\
Batch size (prompts per step)
& $64$
& \ctrlc{Shared} \\
Micro-batch size per device
& $4$
& \ctrlc{Shared} \\
Precision
& bfloat16
& \ctrlc{Shared} \\

\midrule
\multicolumn{3}{@{}l}{\emph{Rollout generation}} \\
Rollout temperature
& $0.6$
& \ctrlc{Shared} \\
Rollout top-$p$
& $1.0$
& \ctrlc{Shared} \\
Rollouts per prompt
& $1$
& \ctrlc{Shared} \\
Max prompt length
& $4096$
& \ctrlc{Shared} \\
Max response length
& $8192$
& \ctrlc{Shared} \\

\midrule
\multicolumn{3}{@{}l}{\emph{Distillation objective}} \\
Divergence
& Reverse KL, $D_{\mathrm{KL}}(p_{\theta,t}\Vert q_{k,t})$
& \ctrlc{Shared} \\
Candidate-set size $C$
& $128$ (student top-$C$)
& \ctrlc{Shared} \\
Student renormalization
& On the candidate set
& \ctrlc{Shared} \\
Teacher-pool size $K$
& $3$ (math, code, IF)
& \ctrlc{Shared} \\
Loss aggregation
& Mean of per-response token means
& \ctrlc{Shared} \\

\midrule
\multicolumn{3}{@{}l}{\emph{Allocation}} \\
\oursc{Sharpness exponent $\gamma$}
& \oursc{$7$; sweep: $\{3,5,7,9\}$}
& \oursc{TrustMOPD} \\

\midrule
\multicolumn{3}{@{}l}{\emph{Offline calibration of $\mu_k$
(Algorithm~\ref{alg:calibration})}} \\
Calibration batches
& $8$, generated by $\pi_{\theta_0}$
& Calibrated \\
Domain labels used
& None
& Calibrated \\
Calibration configuration
& Matched model pool, prompt set, response-length limit, and $C$
& Calibrated \\
Update schedule
& Estimated before training, then frozen
& Calibrated \\

\midrule
\multicolumn{3}{@{}l}{\emph{Evaluation and reporting}} \\
AIME 25/26
& $16$ samples, temperature $0.6$
& \ctrlc{Shared} \\
LiveCodeBench v5/v6
& $5$ samples, temperature $0.6$
& \ctrlc{Shared} \\
IFEval / IFBench
& $1$ sample, temperature $0.6$
& \ctrlc{Shared} \\
Checkpoint reported
& Final; no evaluation-based selection
& \ctrlc{Shared} \\
Training seeds
& $5$
& \ctrlc{Shared} \\

\bottomrule
\end{tabular}
\end{table}

Table~\ref{tab:apd_hparams} summarizes the training, allocation, calibration, and evaluation settings.
The shared training configuration follows the underlying MOPD implementation and is kept fixed across the compared distillation methods, without method-specific tuning.
The allocation exponent is set to $\gamma=7$ in the main experiments and varied over $\{3,5,7,9\}$ in the sensitivity analysis reported in Appendix~\ref{apd:gamma}.

Calibration settings apply to TrustMOPD and the proxy ablations that use calibrated scores.
For each model setup and prompt set, we estimate $\mu_k$ using the response-length limit and candidate-set size of the corresponding training configuration.
Changes to either setting require recalibration using Algorithm~\ref{alg:calibration}.

\subsection{Algorithms}
\label{apd:exp-algorithms}

\begin{algorithm}[!t]
\caption{Offline calibration of teacher scales}
\label{alg:calibration}
\begin{algorithmic}[1]
\Require Initial student $\pi_{\theta_0}$;
teachers $\{\pi_k\}_{k=1}^{K}$;
reference $\pi_{\mathrm{ref}}$;
unlabeled prompts $\mathcal{D}_{\mathrm{cal}}$;
candidate-set size $C$;
number of calibration batches $B_{\mathrm{cal}}=8$
\Ensure Frozen calibration constants $\{\mu_k\}_{k=1}^{K}$

\State $S_k\gets 0$ for all $k$; $N\gets 0$
\For{$b=1,\ldots,B_{\mathrm{cal}}$}
  \State Sample a prompt batch $\mathcal{B}_b$
  from $\mathcal{D}_{\mathrm{cal}}$
  \For{each prompt $x_i\in\mathcal{B}_b$}
    \State Generate $y_i\sim\pi_{\theta_0}(\cdot\mid x_i)$
    \For{$t=1,\ldots,|y_i|$}
      \State $s_t\gets(x_i,y_{i,<t})$;
      $\mathcal{A}_t\gets
      \operatorname{TopC}(\pi_{\theta_0}(\cdot\mid s_t))$
      \State Compute teacher distributions $\{q_{k,t}\}_{k=1}^{K}$
      and reference $q_{\mathrm{ref},t}$, normalized on $\mathcal{A}_t$
      \State $S_k\gets S_k+
      \left\|\operatorname{center}
      (\log q_{k,t}-\log q_{\mathrm{ref},t})\right\|_2$
      for all $k$
      \State $N\gets N+1$
    \EndFor
  \EndFor
\EndFor
\State $\mu_k\gets S_k/N$ for all $k$
\State \Return $\{\mu_k\}_{k=1}^{K}$
\end{algorithmic}
\end{algorithm}

\begin{algorithm}[!t]
\caption{TrustMOPD training}
\label{alg:trustmopd}
\begin{algorithmic}[1]
\Require Student $\pi_\theta$;
fixed teachers $\{\pi_k\}_{k=1}^{K}$;
fixed reference $\pi_{\mathrm{ref}}$;
unlabeled prompts $\mathcal{D}$;
frozen scales $\{\mu_k>0\}_{k=1}^{K}$
from Algorithm~\ref{alg:calibration};
exponent $\gamma>0$;
candidate-set size $C$

\For{each training step}
  \State Sample a prompt batch $\mathcal{B}$ from $\mathcal{D}$
  \For{each prompt $x_i\in\mathcal{B}$}
    \State Generate $y_i\sim\pi_\theta(\cdot\mid x_i)$
    \For{$t=1,\ldots,|y_i|$}
      \State $s_t\gets(x_i,y_{i,<t})$;
      $\mathcal{A}_t\gets
      \operatorname{TopC}(\pi_\theta(\cdot\mid s_t))$
      \State Compute student distribution $p_{\theta,t}$,
      normalized on $\mathcal{A}_t$
      \State Compute teacher distributions $\{q_{k,t}\}_{k=1}^{K}$
      and reference $q_{\mathrm{ref},t}$, normalized on $\mathcal{A}_t$
      \State $\rho_{k,t}\gets
      \left\|\operatorname{center}
      (\log q_{k,t}-\log q_{\mathrm{ref},t})\right\|_2$
      for all $k$
      \State $w_{k,t}\gets
      \operatorname{stopgrad}\!\left(
      \frac{(\rho_{k,t}/\mu_k)^\gamma}
      {\sum_{j=1}^{K}(\rho_{j,t}/\mu_j)^\gamma}
      \right)$ for all $k$
      \State $\widetilde q_{\mathrm{mix},t}\gets
      \operatorname{softmax}_{\mathcal{A}_t}\!\left(
      \sum_{k=1}^{K}w_{k,t}\log q_{k,t}
      \right)$
      \State $\ell_{i,t}(\theta)\gets
      D_{\mathrm{KL}}\!\left(
      p_{\theta,t}\Vert\widetilde q_{\mathrm{mix},t}
      \right)$
    \EndFor
    \State $\mathcal{L}_i(\theta)\gets
    |y_i|^{-1}\sum_{t=1}^{|y_i|}\ell_{i,t}(\theta)$
  \EndFor
  \State $\mathcal{L}(\theta)\gets
  |\mathcal{B}|^{-1}\sum_{x_i\in\mathcal{B}}\mathcal{L}_i(\theta)$
  \State Update $\theta$ using $\nabla_\theta\mathcal{L}(\theta)$
\EndFor
\end{algorithmic}
\end{algorithm}

Algorithm~\ref{alg:calibration} estimates frozen teacher scales by averaging displacement magnitudes over all response tokens generated by the initial student.
Algorithm~\ref{alg:trustmopd} trains the evolving student using these scales, averaging the distillation loss first within each response and then across the batch.
Both procedures use unlabeled prompts.
During differentiation, sampled trajectories and candidate sets are held fixed, and teacher scores and allocation weights are detached.
Token loops describe computations at each state; model scoring can be batched across states.

\section{Additional Experimental Results}
\label{apd:results}

This appendix provides detailed numerical results supporting the analyses in \S\ref{sec:exp-allocation} and \S\ref{sec:exp-ablation}.
The first two subsections report calibrated-score diagnostics and calibration constants, both obtained without distillation training.
\S\ref{apd:alloc}--\ref{apd:gamma} present detailed results for supervision allocation, ablations, sensitivity analyses, and transfer experiments, including per-domain scores and standard deviations.

\subsection{The calibrated score before any distillation}
\label{apd:diag}

\begin{table}[!t]
\centering
\small
\setlength{\tabcolsep}{10pt}
\caption{
\textbf{Calibrated teacher scores before distillation.}
Scores are computed from $256$ initial-student rollouts per domain on held-out prompts.
Entries report means over response states, with standard deviations across prompts.
Rows identify teachers, and columns identify prompt domains.
Within each column, the highest mean score is
\colorbox{oursblue}{tinted}.
}
\vspace{1mm}
\label{tab:apd_rmatrix}
\begin{tabular}{lccc}
\toprule
& \multicolumn{3}{c}{Prompt domain} \\
\cmidrule(lr){2-4}
Teacher & Mathematics & Code & Instruction following \\
\midrule
RL-math & \oursc{$1.120$\std{0.210}} & $0.987$\std{0.076} & $0.541$\std{0.223} \\
RL-code & $0.846$\std{0.221} & \oursc{$1.343$\std{0.424}} & $0.472$\std{0.236} \\
RL-IF   & $0.719$\std{0.067} & $0.714$\std{0.072} & \oursc{$6.868$\std{3.812}} \\
\bottomrule
\end{tabular}
\end{table}

Table~\ref{tab:apd_rmatrix} evaluates the calibrated scores $r_{k,t}$ of \eqref{eq:method-scale} on held-out prompts using rollouts from the initial student.
This complements \textbf{Obs.\ 3} in \S\ref{sec:exp-allocation}, which examines allocation weights from a trained TrustMOPD run.

Within each prompt-domain column, the matching specialist has the highest mean calibrated score.
The separation is particularly pronounced for instruction-following prompts, where the instruction-following teacher scores $6.868$, compared with $0.541$ and $0.472$ for the mathematics and code teachers.
These results show that the allocation signal already exhibits domain-aligned structure before distillation, with scores measured relative to each teacher's calibration mean.

\subsection{Calibration constants}
\label{apd:scales}

\begin{table}[!t]
\centering
\small
\setlength{\tabcolsep}{8pt}
\caption{
\textbf{Calibration constants in the default experiments.}
Values are the $\mu_k$ of \eqref{eq:method-scale} for the SmolLM3-3B teacher pool on the two training prompt sets.
Both settings use the same initial student and a maximum response length of $8192$ tokens.
}
\vspace{1mm}
\label{tab:apd_scales}
\begin{tabular}{lccc}
\toprule
Training prompt set
& $\mu_{\mathrm{math}}$
& $\mu_{\mathrm{code}}$
& $\mu_{\mathrm{IF}}$ \\
\midrule
\textsc{SingleCap} & $3.640$ & $7.482$ & $1.722$ \\
\textsc{MultiCap}  & $3.809$ & $8.223$ & $1.269$ \\
\bottomrule
\end{tabular}
\end{table}

Table~\ref{tab:apd_scales} reports the calibration constants $\mu_k$ used in the default experiments.
The code teacher's mean displacement is approximately $4.3$ times that of the instruction-following teacher on \textsc{SingleCap} and $6.5$ times on \textsc{MultiCap}.
Dividing by $\mu_k$ expresses each teacher's displacement relative to its own calibration baseline.

The constants also differ between the two training prompt sets under the same model setup.
We therefore estimate them separately for each model setup and training prompt set, following \S\ref{apd:exp-training}, including in the backbone and data-source transfer experiments.

\subsection{Allocation analysis}
\label{apd:alloc}

\begin{table}[!t]
\centering
\small
\setlength{\tabcolsep}{5pt}
\renewcommand{\arraystretch}{1.6}
\caption{
\textbf{Rollout-level allocation statistics underlying Figure~\ref{fig:allocation-analysis}(b--d).}
Teacher columns report group means of $\bar{w}_{i,k}=|y_i|^{-1}\sum_t w_{i,k,t}$; each row sums to one across teachers, up to rounding.
\colorbox{oursblue}{Tinted} entries identify domain-matched teachers for \textsc{SingleCap} and the two composition-matched teachers for \textsc{MultiCap}; \emph{Relevant mass} is their combined mean weight.
For \textsc{SingleCap}, we also report the mean and positive fraction of the margin $m_i=\bar{w}_{i,d_i}-\max_{k\neq d_i}\bar{w}_{i,k}$, where $d_i$ is the prompt domain, as in Figure~\ref{fig:allocation-analysis}(c).
The \colorbox{ctrlgray}{shaded} row gives uniform allocation; its relevant mass is shown for \textsc{SingleCap} / \textsc{MultiCap}, and its margin statistics apply to \textsc{SingleCap}.
Group labels are used only for this analysis, not for allocation.
}
\vspace{1mm}
\label{tab:apd_weights}
\begin{tabular}{@{}c@{\hspace{5pt}}lcccccc@{}}
\toprule
\multirow{2.4}{*}{Data}
& \multirow{2.4}{*}{Prompt group}
& \multicolumn{3}{c}{Mean weight $\bar{w}_{i,k}$}
& \multirow{2.4}{*}{\shortstack{Relevant\\mass}}
& \multirow{2.4}{*}{\shortstack{Mean\\margin $m_i$}}
& \multirow{2.4}{*}{$\Pr[m_i>0]$} \\
\cmidrule(lr){3-5}
& & RL-math & RL-code & RL-IF & & & \\
\midrule
& \ctrlc{Uniform allocation $1/K$}
& \ctrlc{$0.33$} & \ctrlc{$0.33$} & \ctrlc{$0.33$}
& \ctrlc{$0.33$ / $0.67$}
& \ctrlc{$0.00$} & \ctrlc{$0\%$} \\
\midrule
\multirow{3}{*}{\rotatebox[origin=c]{90}{\textsc{SingleCap}}}
& Mathematics
& \oursc{$0.69$} & $0.19$ & $0.12$
& \oursc{$0.69$} & $0.46$ & $98\%$ \\
& Code
& $0.30$ & \oursc{$0.65$} & $0.05$
& \oursc{$0.65$} & $0.34$ & $91\%$ \\
& Instruction following
& $0.00$ & $0.00$ & \oursc{$0.99$}
& \oursc{$0.99$} & $0.99$ & $100\%$ \\
\midrule
\multirow{3}{*}{\rotatebox[origin=c]{90}{\textsc{MultiCap}}}
& Math $+$ IF
& \oursc{$0.51$} & $0.11$ & \oursc{$0.38$}
& \oursc{$0.89$} & --- & --- \\
& Code $+$ IF
& $0.25$ & \oursc{$0.47$} & \oursc{$0.28$}
& \oursc{$0.75$} & --- & --- \\
& Math $+$ Code
& \oursc{$0.27$} & \oursc{$0.46$} & $0.27$
& \oursc{$0.73$} & --- & --- \\
\bottomrule
\end{tabular}
\end{table}

Table~\ref{tab:apd_weights} reports the rollout-level allocation statistics underlying Figure~\ref{fig:allocation-analysis}(b--d).
On \textsc{SingleCap}, the domain-matched teacher receives the largest rollout-averaged weight on $98\%$, $91\%$, and $100\%$ of mathematics, code, and instruction-following rollouts, respectively.
The alignment therefore holds across most individual rollouts as well as in the group means.
On \textsc{MultiCap}, the two teachers corresponding to each capability composition receive a combined mean weight of $0.89$, $0.75$, and $0.73$ for mathematics--instruction following, code--instruction following, and mathematics--code, respectively.

Figure~\ref{fig:apd-alloc-tokens} complements these aggregate statistics with token-level examples from all three compositions.
The examples show allocation varying across positions within a response. 
The response-level ablation in Table~\ref{tab:apd_ablation} evaluates the contribution of this positional variation.

\begin{figure}[!t]
\centering
\includegraphics[width=\linewidth]{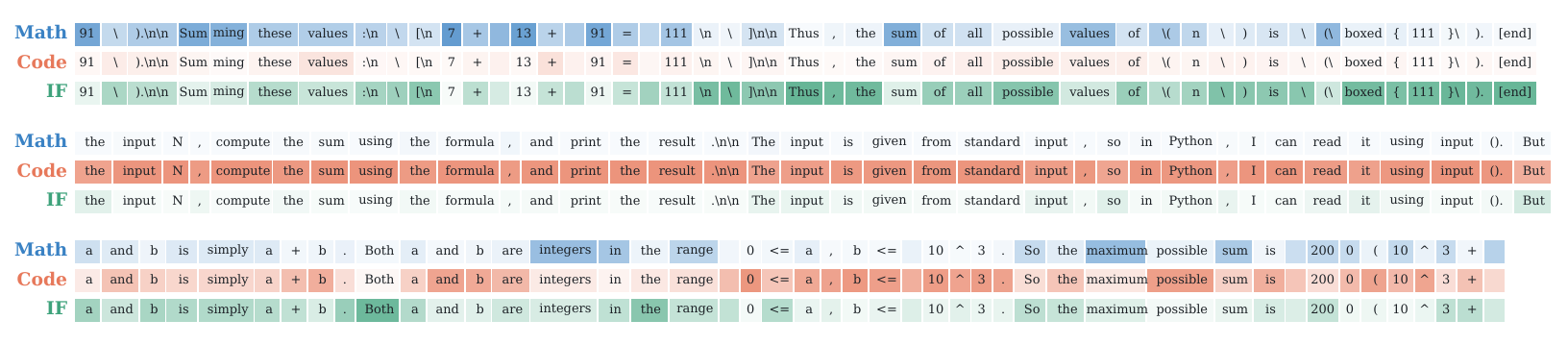}
\caption{
\textbf{Token-level allocation across all three
\textsc{MultiCap} compositions.}
This figure extends the mathematics--instruction-following
example in Figure~\ref{fig:allocation-analysis}(e)
with examples covering all three compositions.
Each block displays the same response tokens once per teacher;
darker shading indicates a higher allocation weight
at that position.
}
\label{fig:apd-alloc-tokens}
\end{figure}

\subsection{Core component ablations}
\label{apd:ablation}

\subsubsection{Definitions of the ablated variants}
\label{apd:ablation-defs}

This subsection defines the ablation variants in Figure~\ref{fig:component-ablation}.
Throughout, $q_{k,t}$ and $p_{\theta,t}$ denote teacher and student distributions normalized on the student's top-$C$ candidate set $\mathcal{A}_t$, following
\S\ref{apd:exp-training}. Full TrustMOPD combines
\eqref{eq:method-displacement},
\eqref{eq:method-scale}, and
\eqref{eq:method-weight}:
\begin{equation}
\underbrace{\rho_{k,t}}_{\text{displacement}}
\ \longrightarrow\
\underbrace{r_{k,t}=\rho_{k,t}/\mu_k}_{\text{calibration}}
\ \longrightarrow\
\underbrace{w_{k,t}}_{\text{allocation}}.
\label{eq:ablation-pipeline}
\end{equation}
Each variant changes the component specified below while retaining the shared training settings in
\S\ref{apd:exp-training}.

\paragraph{Proxy substitutions.}
These variants replace the displacement score
$\rho_{k,t}$ with one of three alternative signals:
\begin{align}
\text{teacher likelihood:}\quad
&\rho^{\mathrm{lik}}_{k,t}:=q_{k,t}(y_t),
\label{eq:proxy-lik}\\
\text{negative entropy:}\quad
&\rho^{\mathrm{negH}}_{k,t}:=\log C-H(q_{k,t}),
\qquad
H(q):=-\!\!\sum_{v\in\mathcal{A}_t}\!q(v)\log q(v),
\label{eq:proxy-negent}\\
\text{teacher--student divergence:}\quad
&\rho^{\mathrm{tsd}}_{k,t}:=
D_{\mathrm{KL}}\bigl(q_{k,t}\,\Vert\,p_{\theta,t}\bigr).
\label{eq:proxy-tsd}\end{align}
Here, $y_t$ is the token sampled from the student at position $t$.
The negative-entropy score is shifted by $\log C$ to make it nonnegative.
The divergence score uses the teacher-to-student KL direction, opposite to the reverse-KL distillation loss in \eqref{eq:weighted-mopd}.

Each alternative score is calibrated separately using Algorithm~\ref{alg:calibration}, with $\rho_{k,t}$ replaced by $\rho^{\mathrm{alt}}_{k,t}$.
The resulting scale $\mu_k^{\mathrm{alt}}$ gives $r^{\mathrm{alt}}_{k,t}=\rho^{\mathrm{alt}}_{k,t}/\mu_k^{\mathrm{alt}}$ an empirical mean of one over the calibration tokens.
The calibrated scores are then converted into allocation weights using \eqref{eq:method-weight} with the same exponent $\gamma$.

\paragraph{Uncalibrated scores.}
``Without $\mu_k$'' uses raw displacement scores
without teacher-specific calibration:
\begin{equation}
\mu_k:=1\quad\text{for every }k,
\qquad
r_{k,t}=\rho_{k,t},
\qquad
w_{k,t}
=
\frac{\rho_{k,t}^{\,\gamma}}
{\sum_{j=1}^{K}\rho_{j,t}^{\,\gamma}}.
\label{eq:no-mu-ablation}
\end{equation}

\paragraph{Response-level allocation.}
This variant computes the token-level weights of full
TrustMOPD using \eqref{eq:method-weight}, then replaces
them with their mean within each sampled response:
\begin{equation}
\bar{w}_{i,k}
:=
\frac{1}{|y_i|}
\sum_{t=1}^{|y_i|}w_{i,k,t},
\qquad
w^{\mathrm{resp}}_{i,k,t}
:=
\bar{w}_{i,k}
\quad\text{for all }t\le|y_i|.
\label{eq:response-level-ablation}
\end{equation}
The averaged weights remain nonnegative and sum to one
across teachers.
For every teacher $k$ and rollout $i$, they also satisfy
\begin{equation}
\sum_{t=1}^{|y_i|}w^{\mathrm{resp}}_{i,k,t}
=
\sum_{t=1}^{|y_i|}w_{i,k,t}.
\label{eq:response-level-mass}
\end{equation}
Thus, on a given rollout, this transformation preserves
each teacher's total allocation weight while removing
its variation across token positions.
The ablation tests the contribution of within-response
allocation.

\paragraph{The control.}
\textsc{Uniform}, $w_{k,t}=1/K$, is defined in
\S\ref{apd:exp-baselines} and included in
Table~\ref{tab:apd_ablation} as the reference for assessing
whether each alternative scoring rule improves
performance over equal teacher weighting.

\subsubsection{Results}
\label{apd:ablation-results}

\begin{table}[!ht]
\centering
\footnotesize
\setlength{\tabcolsep}{3.2pt}
\caption{
\textbf{Core component ablations: absolute scores.}
Results correspond to Figure~\ref{fig:component-ablation}; variant definitions are given in \S\ref{apd:ablation-defs}.
Scores are means $\pm$ standard deviations over five training seeds, evaluated at the final checkpoint.
Within each training-set block, $\Delta_{\mathrm{full}}$ is the difference in mean overall score from full TrustMOPD.
}
\label{tab:apd_ablation}
\begin{tabular}{@{}c@{\hspace{5pt}}lccccc@{}}
\toprule
Data & Ablated component & Math & Code & IF & Overall $\uparrow$ & $\Delta_{\mathrm{full}}$ \\
\midrule
\multirow{7}{*}{\rotatebox[origin=c]{90}{\textsc{SingleCap}}}
& \ctrlc{\textsc{Uniform} \ $w_{k,t}=1/K$ (control)}
& \ctrlc{$20.14$\std{0.23}} & \ctrlc{$18.67$\std{0.38}} & \ctrlc{$45.04$\std{1.23}}
& \ctrlc{$27.95$\std{0.59}} & \ctrlc{$-2.09$} \\
& \oursc{TrustMOPD (full)}
& \oursc{$21.91$\std{0.61}} & \oursc{$20.67$\std{0.75}} & \oursc{$47.56$\std{0.48}}
& \oursc{$\mathbf{30.04}$\std{0.47}} & \oursc{$\phantom{-}0.00$} \\
\cmidrule(l){2-7}
& \quad proxy $\to$ teacher likelihood $q_{k,t}(y_t)$
& $20.02$\std{0.92} & $18.90$\std{0.67} & $46.29$\std{0.57}
& $28.41$\std{0.42} & $-1.63$ \\
& \quad proxy $\to$ negative entropy $-H(q_{k,t})$
& $19.96$\std{0.72} & $18.79$\std{0.69} & $46.47$\std{0.72}
& $28.41$\std{0.25} & $-1.63$ \\
& \quad proxy $\to$ teacher--student $D_{\mathrm{KL}}(q_{k}\Vert p_{\theta})$
& $19.66$\std{0.84} & $17.82$\std{0.37} & $46.40$\std{0.32}
& $27.96$\std{0.20} & $-2.08$ \\
\cmidrule(l){2-7}
& \quad without $\mu_k$ \ (set $\mu_k=1$)
& $20.79$\std{0.50} & $20.39$\std{0.64} & $47.99$\std{0.82}
& $29.73$\std{0.49} & $-0.31$ \\
\cmidrule(l){2-7}
& \quad response-level $\bar{w}_{i,k}\to\widetilde{w}_{i,k}$
& $21.57$\std{1.29} & $20.42$\std{1.10} & $47.97$\std{1.02}
& $29.99$\std{0.74} & $-0.05$ \\
\midrule
\multirow{7}{*}{\rotatebox[origin=c]{90}{\textsc{MultiCap}}}
& \ctrlc{\textsc{Uniform} \ $w_{k,t}=1/K$ (control)}
& \ctrlc{$20.20$\std{0.43}} & \ctrlc{$18.87$\std{0.52}} & \ctrlc{$43.88$\std{1.42}}
& \ctrlc{$27.65$\std{0.53}} & \ctrlc{$-2.76$} \\
& \oursc{TrustMOPD (full)}
& \oursc{$22.92$\std{1.40}} & \oursc{$20.69$\std{0.50}} & \oursc{$47.62$\std{1.41}}
& \oursc{$\mathbf{30.41}$\std{0.60}} & \oursc{$\phantom{-}0.00$} \\
\cmidrule(l){2-7}
& \quad proxy $\to$ teacher likelihood $q_{k,t}(y_t)$
& $20.44$\std{0.81} & $19.10$\std{0.83} & $45.46$\std{2.06}
& $28.33$\std{0.88} & $-2.08$ \\
& \quad proxy $\to$ negative entropy $-H(q_{k,t})$
& $20.54$\std{1.01} & $19.47$\std{0.65} & $46.78$\std{0.77}
& $28.93$\std{0.38} & $-1.48$ \\
& \quad proxy $\to$ teacher--student $D_{\mathrm{KL}}(q_{k}\Vert p_{\theta})$
& $19.04$\std{0.72} & $17.11$\std{0.64} & $45.91$\std{2.21}
& $27.35$\std{0.85} & $-3.06$ \\
\cmidrule(l){2-7}
& \quad without $\mu_k$ \ (set $\mu_k=1$)
& $20.79$\std{0.68} & $20.78$\std{0.39} & $45.83$\std{1.95}
& $29.13$\std{0.54} & $-1.28$ \\
\cmidrule(l){2-7}
& \quad response-level $\bar{w}_{i,k}\to\widetilde{w}_{i,k}$
& $21.21$\std{0.36} & $19.56$\std{0.78} & $47.71$\std{1.45}
& $29.49$\std{0.79} & $-0.92$ \\
\bottomrule
\end{tabular}
\end{table}

Table~\ref{tab:apd_ablation} reports the absolute scores corresponding to Figure~\ref{fig:component-ablation}, which shows changes relative to full TrustMOPD.
All ablated variants have lower mean overall scores than full TrustMOPD on both training sets, although some individual domain scores improve.

All three alternative scoring rules underperform reference-relative displacement after separate calibration.
Teacher--student divergence performs close to \textsc{Uniform} on \textsc{SingleCap} and below it on \textsc{MultiCap}. 
These comparisons support the choice of displacement over the tested alternative signals.

Response-level allocation leaves the mean overall score nearly unchanged on \textsc{SingleCap} ($0.05$ points lower), but reduces it by $0.92$ points on \textsc{MultiCap}.
This pattern suggests a larger benefit from within-response allocation when training prompts combine multiple capabilities.

\subsection{Allocation sharpness and transfer}
\label{apd:gamma}

\begin{table}[!t]
\centering
\footnotesize
\setlength{\tabcolsep}{3.2pt}
\caption{
\textbf{Sensitivity to the allocation sharpness exponent
$\gamma$.}
Absolute scores and allocation concentration corresponding
to Figure~\ref{fig:robustness-generalization}(a,b).
Within each training-set block,
$\Delta_{\mathrm{Uni}}$ is the difference in mean overall
score from \textsc{Uniform}.
Concentration is measured by
$\mathbb{E}_{s_t}[\max_k w_{k,t}]$,
which equals $1/K$ for \textsc{Uniform}.
Scores are means $\pm$ standard deviations over five
training seeds, evaluated at the final checkpoint.
}
\vspace{1mm}
\label{tab:apd_gamma}
\begin{tabular}{@{}c@{\hspace{5pt}}lcccccc@{}}
\toprule
Data & Allocation rule & Math & Code & IF & Overall $\uparrow$ & $\Delta_{\mathrm{Uni}}$ & $\mathbb{E}[\max_k w_k]$ \\
\midrule
\multirow{5}{*}{\rotatebox[origin=c]{90}{\textsc{SingleCap}}}
& \ctrlc{\textsc{Uniform} \ $w_{k,t}=1/K$}
& \ctrlc{$20.14$\std{0.23}} & \ctrlc{$18.67$\std{0.38}} & \ctrlc{$45.04$\std{1.23}}
& \ctrlc{$27.95$\std{0.59}} & \ctrlc{---} & \ctrlc{$0.33$} \\
& \oursc{TrustMOPD, $\gamma=3$}
& \oursc{$22.36$\std{0.65}} & \oursc{$19.96$\std{0.89}} & \oursc{$47.79$\std{0.98}}
& \oursc{$30.04$\std{0.27}} & \oursc{$+2.09$} & \oursc{$0.66$} \\
& \oursc{TrustMOPD, $\gamma=5$}
& \oursc{$21.65$\std{1.00}} & \oursc{$20.60$\std{0.40}} & \oursc{$47.76$\std{0.90}}
& \oursc{$30.00$\std{0.52}} & \oursc{$+2.05$} & \oursc{$0.76$} \\
& \oursc{TrustMOPD, $\gamma=7$ \emph{(default)}}
& \oursc{$21.91$\std{0.61}} & \oursc{$20.67$\std{0.75}} & \oursc{$47.56$\std{0.48}}
& \oursc{$30.04$\std{0.47}} & \oursc{$+2.09$} & \oursc{$0.83$} \\
& \oursc{TrustMOPD, $\gamma=9$}
& \oursc{$22.24$\std{0.84}} & \oursc{$20.40$\std{0.96}} & \oursc{$48.22$\std{0.80}}
& \oursc{$\mathbf{30.29}$\std{0.40}} & \oursc{$+2.34$} & \oursc{$0.86$} \\
\midrule
\multirow{5}{*}{\rotatebox[origin=c]{90}{\textsc{MultiCap}}}
& \ctrlc{\textsc{Uniform} \ $w_{k,t}=1/K$}
& \ctrlc{$20.20$\std{0.43}} & \ctrlc{$18.87$\std{0.52}} & \ctrlc{$43.88$\std{1.42}}
& \ctrlc{$27.65$\std{0.53}} & \ctrlc{---} & \ctrlc{$0.33$} \\
& \oursc{TrustMOPD, $\gamma=3$}
& \oursc{$21.07$\std{0.97}} & \oursc{$20.13$\std{0.59}} & \oursc{$46.71$\std{1.73}}
& \oursc{$29.31$\std{0.97}} & \oursc{$+1.66$} & \oursc{$0.56$} \\
& \oursc{TrustMOPD, $\gamma=5$}
& \oursc{$21.79$\std{1.07}} & \oursc{$20.28$\std{0.40}} & \oursc{$47.16$\std{1.74}}
& \oursc{$29.74$\std{0.74}} & \oursc{$+2.09$} & \oursc{$0.67$} \\
& \oursc{TrustMOPD, $\gamma=7$ \emph{(default)}}
& \oursc{$22.92$\std{1.40}} & \oursc{$20.69$\std{0.50}} & \oursc{$47.62$\std{1.41}}
& \oursc{$\mathbf{30.41}$\std{0.60}} & \oursc{$+2.76$} & \oursc{$0.74$} \\
& \oursc{TrustMOPD, $\gamma=9$}
& \oursc{$22.24$\std{1.14}} & \oursc{$20.59$\std{0.55}} & \oursc{$46.44$\std{1.01}}
& \oursc{$29.76$\std{0.43}} & \oursc{$+2.11$} & \oursc{$0.78$} \\
\bottomrule
\end{tabular}
\end{table}

Table~\ref{tab:apd_gamma} reports the scores and allocation concentration corresponding to Figure~\ref{fig:robustness-generalization}(a,b).
TrustMOPD achieves higher mean overall scores than \textsc{Uniform} at every tested $\gamma\in\{3,5,7,9\}$ on both training sets.

Increasing $\gamma$ produces more concentrated allocation, but does not consistently improve performance.
From $\gamma=7$ to $\gamma=9$, the mean overall score increases by $0.25$ points on \textsc{SingleCap} and decreases by $0.65$ on \textsc{MultiCap}.
Among the tested values, the default $\gamma=7$ achieves the highest mean score on \textsc{MultiCap} and remains close to the best result on \textsc{SingleCap}.

\begin{table}[!ht]
\centering
\footnotesize
\setlength{\tabcolsep}{3.5pt}

\caption{
\textbf{Transfer across model backbones and prompt sources.}
Absolute scores corresponding to Figure~\ref{fig:robustness-generalization}(c,d).
Shaded rows denote \textsc{Uniform} controls.
Within each model and training-data setting, $\Delta_{\mathrm{Uni}}$ reports the mean overall-score gain over the matched \textsc{Uniform} control.
Scores for trained methods are means $\pm$ standard deviations over five training seeds at the final checkpoint.
}
\vspace{1mm}
\label{tab:apd_transfer}
\begin{tabular}{@{}lcccccc@{}}
\toprule
Data & Allocation rule & Math & Code & IF & Overall $\uparrow$ & $\Delta_{\mathrm{Uni}}$ $\uparrow$ \\
\midrule
\multicolumn{7}{@{}p{\textwidth}@{}}{\emph{Default setting: SmolLM3-3B models and the original
training prompt sets (Table~\ref{tab:main})}} \\
\multicolumn{2}{@{}l}{\ctrlc{\textsc{Init} (student initialization)}}
& \ctrlc{$17.02$} & \ctrlc{$15.97$} & \ctrlc{$41.66$}
& \ctrlc{$24.89$} & \ctrlc{---} \\
\cmidrule{1-7}
\multirow{2}{*}{\textsc{SingleCap}}
& \ctrlc{\textsc{Uniform}}
& \ctrlc{$20.14$\std{0.23}} & \ctrlc{$18.67$\std{0.38}} & \ctrlc{$45.04$\std{1.23}}
& \ctrlc{$27.95$\std{0.59}} & \ctrlc{---} \\
& \oursc{TrustMOPD}
& \oursc{$21.91$\std{0.61}} & \oursc{$20.67$\std{0.75}} & \oursc{$47.56$\std{0.48}}
& \oursc{$30.04$\std{0.47}} & \oursc{$+2.09$\std{0.50}} \\
\cmidrule{1-7}
\multirow{2}{*}{\textsc{MultiCap}}
& \ctrlc{\textsc{Uniform}}
& \ctrlc{$20.20$\std{0.43}} & \ctrlc{$18.87$\std{0.52}} & \ctrlc{$43.88$\std{1.42}}
& \ctrlc{$27.65$\std{0.53}} & \ctrlc{---} \\
& \oursc{TrustMOPD}
& \oursc{$22.92$\std{1.40}} & \oursc{$20.69$\std{0.50}} & \oursc{$47.62$\std{1.41}}
& \oursc{$30.41$\std{0.60}} & \oursc{$+2.76$\std{0.26}} \\
\midrule
\multicolumn{7}{@{}p{\textwidth}@{}}{\emph{(c) Backbone $\to$ DeepSeek-R1-Distill-Qwen-7B: new student initialization, shared reference, and public RL teachers}} \\
\multicolumn{2}{@{}l}{\ctrlc{\textsc{Init} (its own student initialization)}}
& \ctrlc{$43.02$} & \ctrlc{$26.99$} & \ctrlc{$28.13$}
& \ctrlc{$32.71$} & \ctrlc{---} \\
\cmidrule{1-7}
\multirow{2}{*}{\textsc{SingleCap}}
& \ctrlc{\textsc{Uniform}}
& \ctrlc{$49.67$\std{3.11}} & \ctrlc{$30.00$\std{1.24}} & \ctrlc{$32.83$\std{1.70}}
& \ctrlc{$37.50$\std{1.33}} & \ctrlc{---} \\
& \oursc{TrustMOPD}
& \oursc{$52.87$\std{0.54}} & \oursc{$30.91$\std{0.72}} & \oursc{$32.91$\std{0.46}}
& \oursc{$38.90$\std{0.45}} & \oursc{$+1.40$\std{0.97}} \\
\cmidrule{1-7}
\multirow{2}{*}{\textsc{MultiCap}}
& \ctrlc{\textsc{Uniform}}
& \ctrlc{$50.55$\std{2.13}} & \ctrlc{$30.37$\std{1.03}} & \ctrlc{$31.68$\std{1.31}}
& \ctrlc{$37.53$\std{1.36}} & \ctrlc{---} \\
& \oursc{TrustMOPD}
& \oursc{$52.68$\std{0.71}} & \oursc{$31.42$\std{0.83}} & \oursc{$31.93$\std{1.28}}
& \oursc{$38.68$\std{0.83}} & \oursc{$+1.15$\std{0.66}} \\
\midrule
\multicolumn{7}{@{}p{\textwidth}@{}}{\emph{(d) Prompt source $\to$ independent corpus: student, teachers and \textsc{Init} as in the reference block}} \\
\multirow{2}{*}{\textsc{SingleCap}}
& \ctrlc{\textsc{Uniform}}
& \ctrlc{$19.30$\std{0.62}} & \ctrlc{$18.22$\std{0.47}} & \ctrlc{$44.12$\std{1.51}}
& \ctrlc{$27.22$\std{0.41}} & \ctrlc{---} \\
& \oursc{TrustMOPD}
& \oursc{$21.36$\std{1.44}} & \oursc{$19.69$\std{0.42}} & \oursc{$47.27$\std{1.04}}
& \oursc{$29.44$\std{0.77}} & \oursc{$+2.22$\std{0.61}} \\
\cmidrule{1-7}
\multirow{2}{*}{\textsc{MultiCap}}
& \ctrlc{\textsc{Uniform}}
& \ctrlc{$20.03$\std{0.98}} & \ctrlc{$18.47$\std{0.62}} & \ctrlc{$44.50$\std{1.29}}
& \ctrlc{$27.67$\std{0.46}} & \ctrlc{---} \\
& \oursc{TrustMOPD}
& \oursc{$22.33$\std{2.52}} & \oursc{$19.74$\std{0.48}} & \oursc{$46.78$\std{1.15}}
& \oursc{$29.62$\std{1.08}} & \oursc{$+1.95$\std{0.73}} \\
\bottomrule
\end{tabular}
\end{table}

Table~\ref{tab:apd_transfer} reports absolute scores for the transfer experiments in Figure~\ref{fig:robustness-generalization}(c,d).
The backbone-transfer setting uses DeepSeek-R1-Distill-Qwen-7B as the student initialization and shared reference, together with three public RL-finetuned teachers.
The data-source transfer setting retains the default models and uses independently sourced training prompts.

Relative to the matched \textsc{Uniform} controls,
TrustMOPD improves mean overall scores by $1.40$ and $1.15$ points on the additional backbone, and by $2.22$ and $1.95$ points on the independent training corpus, for \textsc{SingleCap} and \textsc{MultiCap}, respectively.
These gains show that the benefit of adaptive allocation extends to the additional backbone and data-source settings tested here.

\section{Computational Overhead}
\label{apd:cost}

Table~\ref{tab:apd_cost_train} reports GPU-hours and peak GPU memory for the SmolLM3-3B setting on $8\times$H20 GPUs, following the training protocol in \S\ref{apd:exp-training}.

\begin{table}[!t]
\centering
\footnotesize
\setlength{\tabcolsep}{5pt}
\caption{
\textbf{Computational cost of distillation and calibration.}
Measurements use $8\times$H20 GPUs under the protocol in \S\ref{apd:exp-training}. \emph{Scored models} counts distinct teacher and reference models evaluated per step, excluding the student.
For label-routed MOPD, each example uses one teacher, while mixed-domain batches invoke all three teachers.
\emph{Total GPU-h} is elapsed time multiplied by eight, summed over five training seeds and including evaluation. Peak memory is reported per GPU.
$^{\dagger}$Offline calibration is measured separately and performed once per prompt set; its percentage is relative to the reported TrustMOPD GPU-hour total.
}
\vspace{1mm}
\label{tab:apd_cost_train}
\begin{tabular}{@{}lccccc@{}}
\toprule
\multirow{2.4}{*}{Allocation rule}
& \multicolumn{2}{c}{Scored models / step}
& \multirow{2.4}{*}{\shortstack{Total\\GPU-h}}
& \multirow{2.4}{*}{\shortstack{GPU-h vs.\\\textsc{Uniform}}}
& \multirow{2.4}{*}{\shortstack{Peak mem\\(GB/GPU)}} \\
\cmidrule(lr){2-3}
& Teachers & Reference & & & \\
\midrule
\multicolumn{6}{@{}l}{\emph{Per-step training cost}} \\
Single teacher \quad $\mathbf{e}_{k_0}$
& $1$ & $0$ & $149$ & $-16.8\%$ & $69.2$ \\
MOPD \quad $\mathbf{e}_{d(x)}$
& $3$ & $0$ & $187$ & $+4.5\%$ & $79.8$ \\
\ctrlc{\textsc{Uniform} \quad $\tfrac{1}{K}\mathbf{1}$}
& \ctrlc{$3$} & \ctrlc{$0$} & \ctrlc{$179$} & \ctrlc{$\phantom{+}0.0\%$} & \ctrlc{$87.5$} \\
\textsc{Random} \quad $\operatorname{Unif}(\Delta^{K-1})$
& $3$ & $0$ & $183$ & $+2.2\%$ & $88.4$ \\
\oursc{TrustMOPD \quad $\propto\mathbf{r}_t^{\gamma}$}
& \oursc{$3$} & \oursc{$1$} & \oursc{$192$} & \oursc{$+7.3\%$} & \oursc{$92.6$} \\
\midrule
\multicolumn{6}{@{}l}{\emph{One-off cost, incurred once per prompt set before training}} \\
\oursc{Offline calibration of $\mu_k$ (8 batches)}
& \oursc{$3$} & \oursc{$1$} & \oursc{$1.3$} & \oursc{$+0.7\%^{\dagger}$} & \oursc{$92.6$} \\
\bottomrule
\end{tabular}
\end{table}

\paragraph{Training cost.}
Relative to \textsc{Uniform} and \textsc{Random}, which already score all $K$ teachers, TrustMOPD adds reference scoring and token-level allocation computations.
Compared with \textsc{Uniform}, the measured training-and-evaluation GPU cost increases by $7.3\%$, and peak memory increases by $5.8\%$.
Offline calibration costs an additional $1.3$ GPU-hours per prompt set, with the resulting scales reused across training seeds.

\paragraph{Deployment.}
Deployment uses only the distilled student.
Teachers, the reference, and the allocation procedure are not required, so TrustMOPD introduces no additional inference-time overhead relative to serving the same student architecture.

\section{Case Studies of Token-Level Supervision}
\label{apd:case}
We examine two positions in \textsc{MultiCap} rollouts where teacher preferences can be assessed using explicit local criteria: the result of a subtraction in Figure~\ref{fig:case-math}, and compliance with a no-comma constraint in Figure~\ref{fig:case-if}.
At each position, we compare token-level allocation with the response-level control defined in \eqref{eq:response-level-ablation}.

\paragraph{Shared setup.}
Both cases use SmolLM3-3B rollouts from the initial student at step 0.
Scores and weights are recomputed offline using the training implementation, with $\gamma=7$, $C=128$, and the frozen \textsc{MultiCap} calibration constants.
All reported probabilities and top-1 tokens are evaluated within the student's candidate set $\mathcal{A}_t$.
The target $\tilde q_w$ is the normalized geometric mixture in \eqref{eq:apd-normalized-mixture}; the response-level control uses the same rollout's mean weights $\bar w_k$.

\newcommand{\okmark}{\textcolor{okgreen}{\scriptsize\cmark}}
\newcommand{\nomark}{\textcolor{badred}{\scriptsize\xmark}}

\begin{figure*}[!t]
\centering
\begin{tcolorbox}[colback=white, colframe=black!70, boxrule=0.6pt,
  fonttitle=\bfseries\small, coltitle=white, colbacktitle=black!70,
  top=4pt, bottom=4pt, left=5pt, right=5pt,
  title={Case 1: local arithmetic supervision \hfill \normalfont SmolLM3-3B, step 0}]
{\footnotesize\textbf{Problem.} \textit{56 lines are drawn on a plane such that no three are concurrent. If the lines intersect at exactly 594 points, what is the maximum number of them that could have the same slope?}\par}
\vspace{3pt}
\noindent\colorbox{gray!10}{\parbox{\dimexpr\linewidth-2\fboxsep\relax}{\footnotesize
\textbf{Student prefix} (token $2526$, inside the reasoning block):\par
\texttt{\ldots 1558 - (Sum k\_i\texttwosuperior)/2 = 594}\par
\texttt{Therefore, (Sum k\_i\texttwosuperior)/2 = 1558 -594 =}\par
\vspace{2pt}
\textbf{Local criterion:} $1558-594=964$. Among the $128$ candidates, \texttt{"964"} and \texttt{"96"} match prefixes of this result; $39$ other digit tokens have incompatible prefixes.
}}

\vspace{6pt}
\begin{center}
\renewcommand{\arraystretch}{1.2}
\setlength{\tabcolsep}{4pt}
{\small
\begin{tabular}{@{} l l l c c c c @{}}
\toprule
\textbf{Teacher} & \textbf{Top-1 in $\mathcal{A}_t$} & & \textbf{P(matching prefix)} $\uparrow$ & \textbf{P(incompatible digits)} $\downarrow$ & $r_{k,t}$ & $w_{k,t}$ \\
\midrule
\rowcolor{oursblue}
\textcolor{mathteal}{math} & \texttt{"964"} & \okmark & \textbf{0.538} & \textbf{0.235} & \textbf{3.02} & \textbf{0.999} \\
\textcolor{codeorange}{code} & \texttt{"962"} & \nomark & 0.234 & 0.666 & 0.63 & 0.000 \\
\textcolor{ifgreen}{IF} & \texttt{"962"} & \nomark & 0.124 & 0.716 & 1.19 & 0.001 \\
\midrule
\textit{the student itself} & \texttt{"960"} & \nomark & 0.125 & 0.693 & & \\
\bottomrule
\end{tabular}

\vspace{6pt}
\begin{tabular}{@{} l c c @{}}
\toprule
\textbf{Distilled target $\tilde{q}_w$ at this position} & \textbf{Token-level} & \textbf{Response-level} \\
\midrule
Probability on the verified continuation \texttt{964} $\uparrow$ & \textbf{0.537} & 0.344 \\
\bottomrule
\end{tabular}
}
\end{center}

\vspace{4pt}
{\small
The mathematics teacher's top-1 token is \texttt{"964"},
whereas both other teachers prefer \texttt{"962"}.
Token-level allocation assigns the mathematics teacher
a weight of $0.999$.
The sampled student rollout emitted \texttt{"960"}
at this position.
}
\end{tcolorbox}
\caption{
\textbf{Case 1: supervision at a local arithmetic position.}
The displayed subtraction evaluates to $964$.
Matching-prefix mass sums the probabilities of
\texttt{"964"} and \texttt{"96"};
incompatible-digit mass covers the $39$ digit candidates
that cannot begin this result.
The mathematics teacher has the correct top-1 token
and receives the largest allocation weight.
}
\label{fig:case-math}
\end{figure*}

\begin{figure*}[!t]
\centering
\begin{tcolorbox}[colback=white, colframe=black!70, boxrule=0.6pt,
  fonttitle=\bfseries\small, coltitle=white, colbacktitle=black!70,
  top=4pt, bottom=4pt, left=5pt, right=5pt,
  title={Case 2: local compliance with a no-comma constraint \hfill \normalfont SmolLM3-3B, step 0}]
{\footnotesize \textbf{Problem.} A quadratic-graph question with instruction constraints. The constraint examined here is: \textit{``In your entire response refrain from the use of any commas.''} \par}
\vspace{3pt}
\noindent\colorbox{gray!10}{\parbox{\dimexpr\linewidth-2\fboxsep\relax}{\footnotesize
\textbf{Student prefix} (token $3608$, in the evaluated final response; the scored region starts at token $3384$):\par
\texttt{\ldots 3. **Determining the Vertex**:}\par
\texttt{\ \ \ - The vertex of the quadratic function \textbackslash( q(x) \textbackslash) is at \textbackslash( x = 7.5 \textbackslash}\par
\vspace{2pt}
\textbf{Decidable by the constraint:} \textbf{Local criterion:} a token containing a comma violates the constraint. Seven of the $128$ candidates contain commas, including \texttt{","} and \texttt{"),"}.
}}

\vspace{6pt}
\begin{center}
\renewcommand{\arraystretch}{1.2}
\setlength{\tabcolsep}{4pt}
{\small
\begin{tabular}{@{} l l l c c c @{}}
\toprule
\textbf{Teacher} & \textbf{Top-1 in $\mathcal{A}_t$} & & \textbf{P(forbidden)} $\downarrow$ & $r_{k,t}$ & $w_{k,t}$ \\
\midrule
\textcolor{mathteal}{math} & \texttt{"),"} & \nomark & 0.469 & 1.04 & 0.392 \\
\textcolor{codeorange}{code} & \texttt{"),"} & \nomark & 0.481 & 0.58 & 0.006 \\
\rowcolor{oursblue}
\textcolor{ifgreen}{IF} & \texttt{").\textbackslash n"} & \okmark & \textbf{0.337} & \textbf{1.11} & \textbf{0.602} \\
\midrule
\textit{the student itself} & \texttt{","} & \nomark & 0.493 & & \\
\bottomrule
\end{tabular}

\vspace{6pt}
\begin{tabular}{@{} l c c @{}}
\toprule
\textbf{Distilled target $\tilde{q}_w$ at this position} & \textbf{Token-level} & \textbf{Response-level} \\
\midrule
Probability on forbidden tokens $\downarrow$ & \textbf{0.390} & 0.406 \\
\bottomrule
\end{tabular}
}
\end{center}

\vspace{4pt}
{\small
The instruction-following teacher is the only teacher whose top-1 token contains no comma.
Its token-level weight is $0.602$, compared with a response mean of $\bar w_{\mathrm{IF}}=0.486$.
The token-level target assigns $0.390$ probability to comma-bearing tokens, compared with $0.406$ under response-level allocation.
}
\end{tcolorbox}
\caption{\textbf{Case 2: supervision under a no-comma constraint.}
Forbidden-token mass sums probabilities over the seven comma-bearing candidates.
Among the teachers, the instruction-following teacher assigns the least mass to these tokens and receives the largest allocation weight.
Token-level allocation reduces the target's forbidden-token mass relative to the response-level control.}
\label{fig:case-if}
\end{figure*}

These examples illustrate allocation favoring different specialists at different positions.
In both cases, the token-level target improves the reported local probability metric relative to the response-level control.
They provide qualitative evidence of alignment with local supervision quality; the aggregate performance comparison is reported in Table~\ref{tab:apd_ablation}.

\section{Limitations}
Our findings are established on 3B and 7B backbones with mathematics, code, and instruction-following specialists; larger models and broader domains remain directions for extending this evaluation. MultiCap enables controlled comparisons of capability composition, while its programmatic construction captures only part of the variation in naturally blended user requests. Our experiments support displacement-based allocation within the evaluated specialist pools. The theoretical utility-contrast interpretation assumes globally optimal KL-regularized teachers, while practical RL uses finite-budget optimization. Finally, the current formulation assumes a shared pre-RL reference and incurs reference-scoring overhead during training (Appendix~\ref{apd:cost}), with extensions to teachers of heterogeneous provenance left to future work.

\section{Additional Related Work}
\label{apd:broader-context}

\paragraph{Distillation objectives and infrastructure.}
Related KD methods refine divergence objectives or construct supervision targets~\citep{wu2025akl,ko2024distillm_opd,ko2025distillm2_opd,jin2026entropy_opd,jang2026stableopd_acl,zhang2026opsdl}.
EasyOPD provides infrastructure supporting multiple OPD settings~\citep{sun2026easyopd}.
The training setting introduces further choices: ShortOPD adapts rollout horizons for recovery after structured pruning~\citep{zhang2026shortopd}; for multimodal models, VA-OPD and Vision-OPD exploit visual advantage and crop-conditioned supervision, respectively~\citep{liu2026visual,yuan2026visionopd}, while Video-MOPD combines domain routing with reliability-based example selection~\citep{qin2026videomopd}.
These methods study objective design, trajectory generation, or modality-specific supervision, complementing TrustMOPD's focus on allocating supervision among teachers.

\paragraph{Tokenizer compatibility and input representation.}
Teacher allocation presupposes a usable supervision interface.
SimCT addresses supervision alignment across teacher and student tokenizations~\citep{sun2026simct}.
DVLM addresses this interface through a student-vocabulary projection head and parallel-tokenized inputs~\citep{chen2026dualvocabularylanguagemodelcrosstokenizer}.
Input representation can also affect the capabilities being transferred: SepSeq improves long numerical sequence processing through separator insertion without additional training~\citep{sun2026sepseq}.
TrustMOPD currently assumes a shared vocabulary and fixed input representation; combining allocation with these interface or input-processing methods remains future work.

\paragraph{Capability integration and training data.}
Parameter merging combines model weights or updates, while FuseChat combines distribution fusion with parameter merging~\citep{wortsman2022soups,ilharco2023taskarith,yadav2023ties,wan2025fusechat}.
Within multi-teacher OPD, DF-OPD constructs teacher-generated questions and derives routing labels from teacher identity~\citep{li2026datafreeopd}.
Other approaches integrate capabilities through rationale transfer, distribution matching, or interaction-graph distillation~\citep{tian2025tinyllm,jin2026purification,wan2024fusellm,chen2024magdi}.
The choice of demonstrations provides a complementary route to capability development: LIMI studies learning agency from a small set of curated demonstrations~\citep{xiao2025limi}.
These approaches concern how capabilities or training data are assembled; TrustMOPD studies teacher contributions at sampled student states.

\paragraph{Adaptive objectives and supervision quality.}
Methods can regulate supervision through different signals and at different stages.
Under offline supervision, EWAD uses teacher entropy and inter-teacher agreement to balance teacher and gold supervision~\citep{sumit2026reliabilitygated}.
In preference optimization, $\gamma$-PO uses instance-specific target margins to reduce the influence of ambiguous preference pairs~\citep{sun2025gammapo}.
For OPD, FIRE-OPD combines trajectory filtering with token reweighting~\citep{li2026fireopd}.
For research agents, Apollo combines asynchronous human guidance with action-level filtering to construct training trajectories~\citep{fu2025apollo}.
These choices affect the quality or strength of the available training signal; TrustMOPD addresses how to distribute supervision across fixed teachers at the same student-generated state.

\paragraph{Confidence calibration and exploration in RL.}
CDE~\citep{dai2026cde} identifies calibration collapse during RL with verifiable rewards, where model confidence becomes less predictive of response correctness, and introduces actor log-perplexity and multi-head critic uncertainty as exploration bonuses.
These findings motivate caution when treating raw policy confidence as a reliability signal.
CDE improves exploration during RL training; TrustMOPD allocates supervision among fixed teachers at student-generated states using reference-relative displacement.
The calibration in TrustMOPD normalizes displacement scales across teachers and does not estimate correctness probabilities.

\paragraph{Preference conditioning and capability composition.}
Instruction-following benchmarks such as IFEval and IFBench evaluate explicit, verifiable constraints~\citep{zhou2023IFEval,pyatkin2025IFBench}.
LaMP-Val studies a different form of conditioning, incorporating users' semantic preferences into personalized auction valuation~\citep{sun2025lampval}.
MathIF examines the interaction between mathematical reasoning and instruction constraints~\citep{fu-etal-2026-mathif}, motivating the controlled capability combinations in MultiCap.
These settings distinguish satisfying explicit constraints from modeling user preferences; extending teacher allocation to personalized decision tasks would require task-specific evaluation.

\paragraph{Software-engineering data and evaluation.}
Coding supervision ranges from verifiable synthetic problems, as in KodCode~\citep{xu2025kodcode}, to agent-native mid-training trajectories that preserve interaction context and actual tool observations, as in daVinci-Dev~\citep{zeng2026davincidev}.
Repository-level tasks also require executable environments: SWE-bench evaluates resolution of real GitHub issues~\citep{jimenez2024swebench}, while daVinci-Env constructs software-engineering environments and curates training trajectories~\citep{fu2026davincienv}.
Evaluation granularity changes with the task: LiveCodeBench assesses coding capabilities including code generation and execution~\citep{jain2024livecodebench}, whereas LongCLI-Bench evaluates long-horizon command-line programming with requirement and regression tests~\citep{feng2026longcli}.
These distinctions matter for extending TrustMOPD from response generation to supervision allocation over repeated tool interactions.

\paragraph{Interactive and research-agent evaluation.}
AgentBench assesses reasoning and decision-making across interactive environments~\citep{liu2024agentbench}; AgencyBench evaluates diverse real-world workflows in extended contexts~\citep{li2026agencybench}.
The interaction protocol also determines how capabilities are composed: ReAct interleaves reasoning with actions and environmental observations~\citep{yao2023react}.
InnovatorBench examines such composition at the level of end-to-end AI research tasks~\citep{wu2026innovatorbench}.
Compared with MultiCap's combinations within a response, these settings involve intermediate errors and tool feedback that affect subsequent decisions.
They provide directions for future evaluation; our current experiments do not establish transfer to these agent settings.

\end{document}